\documentclass{alt2019}
\usepackage{times}

\usepackage{karim}

\begin{document}

\title{An Exponential Efron-Stein Inequality for $L_q$ Stable Learning Rules}

 \coltauthor{%
  \Name{Karim Abou-Moustafa}
	\thanks{Now at SAS Inst. Inc, Cary, North Carolina, USA.} 
  \Email{Karim.Abou-Moustafa@sas.com}\\
       \addr Dept. of Computing Science\\
       University of Alberta \\
       Edmonton, AB T6G 2E8, Canada
  \AND
  \Name{Csaba Szepesv\'ari}
	\thanks{On leave at DeepMind, London, UK.} 
  \Email{csaba.szepesvari@gmail.com}\\
       \addr Dept. of Computing Science \\
       University of Alberta\\
       Edmonton, AB T6G 2E8, Canada
}

\maketitle

\begin{abstract}%
There is an accumulating evidence in the literature that \emph{stability of learning algorithms} is a key characteristic 
that permits a learning algorithm to generalize.
Despite various insightful results in this direction, there seems to be an overlooked dichotomy in the type of 
stability-based generalization bounds we have in the literature.
On one hand, the literature seems to suggest that exponential generalization bounds for the estimated risk, which are 
optimal, can be \emph{only} obtained through \emph{stringent}, \emph{distribution independent} and 
\emph{computationally intractable} notions of stability such as \emph{uniform stability}. 
On the other hand, it seems that \emph{weaker} notions of stability such as hypothesis stability, although it is 
\emph{distribution dependent} and more \emph{amenable} to computation, can \emph{only} yield polynomial generalization 
bounds for the estimated risk, which are suboptimal.

In this paper, we address the gap between these two regimes of results.
In particular, the main question we address here is \emph{whether it is possible to derive exponential generalization 
bounds for the estimated risk using a notion of stability that is computationally tractable and distribution dependent, 
but weaker than uniform stability}.
Using recent advances in concentration inequalities, and using a notion of stability that is weaker than uniform stability
but distribution dependent and amenable to computation, we derive an exponential tail bound for the concentration of 
the estimated risk of a hypothesis returned by a \emph{general} learning rule, where the estimated risk is expressed in 
terms of either the resubstitution estimate (empirical error), or the deleted (or, leave-one-out) estimate.
As an illustration we derive exponential tail bounds for ridge regression 
with \emph{unbounded responses}, where we show how stability changes with the tail behavior of the response variables.
\end{abstract}

\begin{keywords}
Generalization bounds, 
algorithmic stability,
the leave one out estimate,
resubstitution estimate,
empirical error,
concentration inequalities,
moments bounds,
tail bounds,
Efron-Stein inequality.
\end{keywords}

\newcommand{\sss}{\scriptscriptstyle}
\newcommand{\Alg}{\falgo{A}}
\newcommand{\Vdel}{V_{\mathtxtxs{DEL}}}
\newcommand{\Vrep}{V_{\mathtxtxs{REP}}}
\newcommand{\Rest}{\est{R}}
\newcommand{\Rdel}{\est{R}_{\mathtxtxs{DEL}}}
\newcommand{\Rres}{\est{R}_{\mathtxtxs{RES}}}
\newcommand{\RdelASn}{\Rdel\bra{\Alg,\fset{S}_n}}
\newcommand{\RiskASnP}{R\bra{\Alg(\fset{S}_n),\fdistri{P}}}

\section{Introduction}
\label{sec:intro}

There is an accumulating evidence in the literature that stability of learning algorithms is a key characteristic that 
permits a learning algorithm to generalize.
The earliest results in this regard are due to \citet{devroy_wagner_79A,devroy_wagner_79B} where they derive 
distribution-free \emph{exponential} generalization bounds for the concentration of the \emph{leave-one-out} estimate, or 
the \emph{deleted} estimate, for $k$ local learning rules. 
Although the notion of stability was not explicitly mentioned in their work, the exponential bounds of 
\citet{devroy_wagner_79A,devroy_wagner_79B} can be seen as relying on the so called \emph{hypothesis stability}; 
a concept due to \citet{kearns_ron_neco_99}.

Various results for different estimates followed the works of \citet{devroy_wagner_79B,devroy_wagner_79A}.
\citet{lugosi_pawlak_94} extended the work of \citet{devroy_wagner_79B,devroy_wagner_79A} to smooth estimates of the error
developed in terms of \aposteriori distribution for the deleted estimate.
\citet{holden_paclikebounds_del_colt96} derived \emph{sanity-check bounds} for the deleted estimate and the \emph{$k$ 
folds cross--validation} (KFCV) estimate using hypothesis stability for few algorithms in the realizable setting.%
\footnote{In particular, \citet{holden_paclikebounds_del_colt96} considered the closure algorithm, and the deterministic 
$1$--inclusion graph prediction strategy.}
Sanity-check bounds are \emph{assurances} that the worst deleted estimate and the worst KFCV estimate will not be 
considerably worse than the \emph{training error} or the \emph{resubstitution} estimate (also known as the empirical error, or empirical risk) \citep{devroy_wagner_79B}.
\citet{kearns_ron_neco_99}, using the notion of \emph{error stability}, give sanity-check bounds for the deleted estimate 
but for more general classes of learning rules (in the unrealizable or agnostic setting).
In particular, they show that if a learning algorithm has a finite VC dimension search space, then the algorithm is 
\emph{error-stable} and its error stability is controlled by the said VC dimension.
Hence, using stability as a complexity measure will not yield worse bounds than using the VC dimension.
Note that error stability is much weaker than hypothesis stability in the sense that hypothesis stability implies error 
stability, and this weakness was necessary to obtain more general sanity-check bounds than those obtained by 
\citet{holden_paclikebounds_del_colt96}.
More recently, \citet{kale_kumar_cv_ms_stability_2011} show that, using a weak notion of stability known as \emph{mean-square} 
stability, the averaging taking place in the KFCV estimation procedure can reduce the variance of the generalization error;
i.e. the averaging in the KFCV estimation procedure can improve the concentration of the estimated error around the expected 
error of the hypothesis returned by the learning rule.

For general learning rules and for \emph{regularized empirical risk minimization} learning rules, 
\citet{bousquet_elisseeff_jmlr_2002} using the notion of \emph{uniform stability}, extended the work of 
\citet{lugosi_pawlak_94} and derived \emph{exponential} generalization bounds for the resubstitution estimate and the 
deleted estimate.
Further generalization results based on uniform stability (or one of its variants) were later obtained in the works of 
\citet{kutin_niyogi_02,sasha_sayan_poggio_2005,mukherjee_advcompmath_2006,shwartz_shamir_jmlr2010}, to name but a few.
In particular, \citet{shwartz_shamir_jmlr2010} showed that a version of uniform stability is key to learnability in the general learning setting with uniformly bounded losses.
These results were reinforced and extended in various directions such as deriving new results for 
randomized learning algorithms \citep{elisseeff_evgeniou_ponttil_jmlr2005},
transfer and meta learning \citep{maurer_jmlr2005},
adaptive data analysis \citep{raef_nissim_sotc2016},
stochastic gradient descent \citep{hardt_recht_singer_icml2016},
structured prediction \citep{london_getoor_jmlr2016},
multi-task learning \citep{zhang_aaai2015},
ranking algorithms \citep{shivaagarwal_niyogi_jmlr2009}, 
as well as in understanding the trade-off between sparsity and stability \citep{xu_mannor_pami_2012}.

Despite these recent advances, and excluding sanity-check bounds, there seems to be an overlooked dichotomy in the type 
of stability-based generalization results.
In particular, the results on stability and generalization can be grouped into two regimes:
\begin{enumerate}
\item \emph{Polynomial} generalization bounds, which are \emph{sub-optimal} and based on hypothesis stability for instance. 
\item \emph{Exponential} generalization bounds, which are \emph{optimal} and based on uniform stability (and its variants).
\end{enumerate}

Comparing \emph{uniform} stability to other notions of stability in the literature, uniform stability is the strongest 
(most demanding) 
notion of stability in the sense that it implies all other notions of stability such as hypothesis stability, error 
stability, and mean-square stability \citep{bousquet_elisseeff_jmlr_2002}.
A learning rule is \emph{uniformly stable} if the change in the prediction loss is small, no matter how the input to the 
learning rule is selected, no matter what value is used as a test example, and no matter which example is removed 
(or replaced) in the input. 

Despite the strength of uniform stability, it is unpleasantly restrictive.
First, unlike other notions of stability (e.g. $L_2$ and $L_1$ stability), uniform stability is a stringent notion of 
stability that is insensitive to the data-generating distribution.
This is problematic since it removes the possibility of studying large classes of learning rules, or even classes of 
problems.
One particularly striking example is binary classification with the zero-one loss.
For this problem, as it was already noted by \cite{bousquet_elisseeff_jmlr_2002}, \emph{no non-trivial algorithm} can be
uniformly $\beta$-stable with $\beta < 1$.
Another example when uniform stability fails is regression with unbounded losses and response variables.%
\footnote{One possibility to fix this is to use a case-based analysis that splits the event space based on whether all examples in the training set are ``small'' may still be used in this case, but this splitting is not without artifacts.
}
Second, as noted earlier, uniform stability is distribution-free and is thus unsuitable for studying finer details of 
learning algorithms.
Computation is another aspect that distinguishes uniform stability from other notions of stability.
While hypothesis, error, and mean-square stability can be estimated using a finite sample, uniform stability is 
computationally intractable.
In other words, although uniform stability yields exponential generalization bounds, these bounds cannot be empirically 
estimated using a finite sample in the spirit of empirical Bernstein bounds for instance 
\citep{audibert_munos_csaba_2007,mnih_csaba_audibet_icml_2008}.

In this research, we are particularly motivated by these previous observations.
That is, on the one hand, the literature seems to suggest that exponential generalization bounds for the estimated risk, 
which are optimal, can be \emph{only} obtained through \emph{stringent}, \emph{distribution independent}, and 
\emph{computationally intractable} notions of stability such as uniform stability (and its variants).
On the other hand, it seems that \emph{weaker} notions of stability such as hypothesis and mean-square stability, 
although they are \emph{distribution dependent} and potentially more \emph{amenable} to computation, can \emph{only} yield 
polynomial generalization bounds for the estimated risk, which are sub-optimal. 

The chief purpose of this paper is to address the gap between these two regimes of results.
In particular, the main question we address here is \emph{whether it is possible to derive exponential generalization 
bounds for the estimated risk using a notion of stability that is computationally tractable, distribution dependent, 
but weaker than uniform stability}.
Our work here gives a positive answer to this question; we show that using recent advances in exponential concentration 
inequalities, and using a notion of stability that is distribution dependent, amenable to computation, but weaker than 
uniform stability, we derive in \cref{theorem:exp_tail_bound_for_del} an exponential tail bound for the concentration of 
the estimated risk of a hypothesis returned by a \emph{general} learning rule, where the estimated risk is developed in 
terms of either the deleted estimate, or the resubstitution estimate (also known as the empirical error).

Two main ingredients that allowed us to bridge the gap between these two regimes of results;
(\emph{i}) recent advances in exponential concentration inequalities, in particular the exponential Efron-Stein 
inequality due to \citet{boucheron_lugosi_massart_2003} and \citet{book_concentration_2013}; and
(\emph{ii}) the elegant notion of $L_q$ stability due to \citet{celisse_guedj_2016} which is distribution 
dependent, weaker than uniform stability, and generalizes hypothesis stability and mean-square stability to higher order 
moments.

Exponential Efron-Stein inequalities aim to bound the deviation of a general function $f$ of $n$ independent input random 
variables (RVs) from its expected value.%
\footnote{The $n$ independent RVs are not necessarily identically distributed.}
The seminal works of \citet{boucheron_lugosi_massart_2003} and \citet{book_concentration_2013} bound this deviation by
means of variance-like terms that measure the sensitivity of $f$ with respect to the \emph{replacement} of one RV from 
the $n$ independent input RVs to $f$, with another independent copy of this RV.
This notion of sensitivity with respect to the replacement of RVs is not suitable for our purposes, nor does it fit 
naturally the empirical estimation of these bounds based on finite datasets. 
As a byproduct of the results presented here, we derive an 
extension of the exponential Efron-Stein inequality when the sensitivity of $f$ is measured with respect to the \emph{removal} 
of one RV from the $n$ independent input RVs to $f$ (see \cref{lemma:exp_efron_stein}).
This notion of sensitivity with respect to the removal of RVs is naturally aligned with the notion of $L_q$ stability, 
and with error estimates such as the deleted estimate and the KFCV estimate.

Efron-Stein inequalities have long been proposed to study the concentration of error estimates.
First, the classic inequality was considered for bounding the variance (e.g., \citep{bousquet_elisseeff_jmlr_2002}).
Soon after \citet{boucheron_lugosi_massart_2003} introduced the variant for higher moments, \citet{sasha_sayan_poggio_2005} 
used this for deriving exponential tail bounds for the so-called almost uniformly stable learning algorithms, replicating 
the results of \citet{kutin_niyogi_02}, who used an extension of McDiarmid's inequality.
More recent use of the higher order moment version is due to \cite{celisse_guedj_2016}, who introduced the distribution-dependent 
$L_q$-stability coefficients and used them to derive bounds on the higher moments of the difference between the deleted 
estimate and the true risk.
\cite{celisse_guedj_2016} used these moment bounds to get exponential tail bounds for the special case of ridge regression.

Our work is closest in spirit to \cite{celisse_guedj_2016}.
However, we provide an alternate route for obtaining exponential tail bounds by providing an exponential Efron-Stein inequality of the ``removal type''  in \cref{lemma:exp_efron_stein}.
This inequality is used to bound the moment-generating function (MGF) of various random variables,
such as the deleted estimate, the resubstitution estimate, or the true risk of the random hypothesis returned by the learning rule.
In each case, the bound is obtained in terms of the MGF of a random variable that corresponds to an average stability quantity of removing a sample.
This latter MGF is bounded by controlling the growth-rate of various $L_q$ stability coefficients,
which leads the final exponential tail bounds.
We obtain such a tail bound for the deleted estimate (\cref{theorem:exp_tail_bound_for_del}),
and also for the resubstitution estimate (\cref{theorem:exp_tail_bound_for_res}).
To control the tail of the resubstitution estimate, we observe that it is not sufficient to control the $L_q$ stability coefficients introduced by 
\cite{celisse_guedj_2016}, but one must also control a related, but distinct quantity, which measures the sensitivity of the algorithm to removing an 
example from the training set when the algorithm is tested on the example that is removed.
We also apply our results to the case of ridge regression with unbounded response variables.
In this case, we obtain the first exponential tail bounds for the deleted estimate (the case of resubstitution estimate is similar, but is not given explicitly).
Since for unbounded response variables, the ridge regression estimator is not uniformly stable,
these tail bounds were out of reach of a straightforward application of previous techniques built on uniform stability.

\section{Setup and Notations}
\label{sec:setup_notations}

We consider learning in Vapnik's framework for risk minimization with bounded losses \citep{book_vapnik_99}: 
A learning problem is specified by the triplet $(\fset{H},\fset{X},\ell)$, where $\fset{H}, \fset{X}$ are sets and
$\ell:\fset{H} \times \fset{X} \to [0,\infty)$.
The set $\fset{H}$ is called the \emph{hypothesis space}, $\fset{X}$ is called the \emph{instance space}, and $\ell$ is
called the \emph{loss function}.
The loss $\ell(h,x)$ indicates how well a hypothesis $h$ explains (or fits) an instance $x \in \fset{X}$.
The learning problem is defined as follows. A learner $\falgo{A}$ sees a sample in the form of a sequence
$\fset{S}_n = (X_1,\dots,X_n) \in \fset{X}^n$ where $(X_i)_i$ is sampled in an independent and identically distributed
(\iid) fashion from some unknown distribution $\fdistri{P}$ and returns a hypothesis $\est{h}_n  = \Alg(\sS_n)\in \sH$
based solely on $X_1,\dots,X_n$.%
\footnote{The set $\fset{X}$ is thus measurable. All functions and sets are assumed and/or can be shown to be 
measurable as needed, saving us from the trouble of mentioning measurability in the rest of the paper.}
The goal of the learner is to pick hypotheses with a small \emph{risk} (defined shortly).

We assume that a learner is able to process samples (or sequences) of different cardinality.
Hence, a learner will be identified with a map $\falgo{A}: \cup_n \fset{X}^n \to \fset{H}$.
We only consider deterministic learning rules in this work; the extension to randomizing learning rules is left for
future work. 

Given a distribution $\fdistri{P}$ on $\fset{X}$, the risk of a \emph{fixed hypothesis} $h\in \fset{H}$ is defined to be 
$R(h,\fdistri{P}) = \expectone{\loss{h,X}}$, where $X \sim \fdistri{P}$.
Since $\fset{S}_n$ is a random quantity, so are $\Alg(\fset{S}_n)$ and $R(\Alg(\sS_n),\fdistri{P})$, the latter of which 
can be also written as
$\EE[\loss{\Alg(\fset{S}_n) , X} | \fset{S}_n]$, where $X\sim \fdistri{P}$ is independent of $\sS_n$.
Ideal learners keep the risk $R(\falgo{A}(\fset{S}_n),\fdistri{P})$ of the hypothesis returned by $\Alg$ ``small'' for 
a wide range of distributions $\fdistri{P}$.
\newline

\noindent
\textbf{$q$-Norm of Random Variables: }
In the sequel, we will heavily rely on the $q$-norm for random variables (RVs).
For a real RV $X$, and for $1 \le q\le +\infty$, the $q$-norm of $X$ is defined as:
$\norm{X}_q \doteq (\expectone{ |X|^q })^{1/q}$, and $\norm{X}_\infty$ is the essential supremum of $|X|$.
Note that for $1 \le q \le p \le +\infty$, these norms satisfy $\norm{\cdot}_q \le \norm{\cdot}_p$.

\subsection{Risk Estimators}
\label{subsec:risk_estimates}

The generalization bounds on the risk usually center on some point-estimate of the random risk
$R(\Alg(\fset{S}_n),\fdistri{P})$.
Many estimators are based on calculating the sample mean of losses in one form or another.
For any fixed hypothesis $h \in \fset{H}$ and dataset $\fset{S}_n$, the sample mean of losses of $h$
against $\fset{S}_n$, also known as  the \emph{empirical risk} of $h$ on $\fset{S}_n$, is given by
\begin{equation}
\Rest(h,\fset{S}_n) = \frac{1}{n} \sum_{i=1}^n \loss{ h, X_i }.
\end{equation}
Plugging $\Alg(\fset{S}_n)$ into $\Rest(\cdot,\fset{S}_n)$ we get the \emph{resubstitution (RES) estimate}, or the 
training error \citep{devroy_wagner_79B}:
$\Rres\bra{\Alg,\fset{S}_n} = \est{R} \bra{\Alg\bra{\fset{S}_n},\fset{S}_n}$.
The resubstitution estimate is often overly ``optimistic'', i.e., it underestimates the actual risk
$R(\Alg(\fset{S}_n),\fdistri{P})$.
The \emph{deleted (DEL) estimate} defined as
\begin{equation}
\label{eq:del_estimate}
\Rdel\bra{\Alg,\sS_n} = \frac{1}{n}\sum_{i=1}^n \loss{ \Alg(\sS_n^{-i} ) , X_i},
\end{equation} 
is a common alternative to the resubstitution estimate that aims to correct for this optimism.
Here, $\sS^{-i}_n = \bra{X_1,\dots,X_{i-1},X_{i+1},\dots,X_n}$, i.e., it is the sequence $\sS_n$ with example $X_i$
removed.
Since $\EE[\loss{\Alg(\sS_n^{-i}),X_i}] = R_{n-1}(\Alg,\fdistri{P})$, then
$\EE[\Rdel(\Alg,\sS_n)] = R_{n-1}(\Alg,\fdistri{P})$.
When the latter is close to $R_n(\Alg,\fdistri{P})$, i.e., $\Alg$ is ``stable'', the deleted estimate may be a good
alternative to the resubstitution estimate \citep{book_devroye_gyorfi_lugosi_1996}.
However, due to the potentially strong correlations between elements of $( \ell(\Alg(\sS_n^{-i}),X_i) )_i$, the
variance of the deleted estimate \emph{may be} significantly higher than that of the resubstitution estimate due 
to the overly redundant information content between $\ell(\Alg(\sS_n^{-i}),X_i)$ and $\ell(\Alg(\sS_n^{-j}),X_j)$ for 
$i \ne j$.
The main goal of this work is to develop a high probability upper bound on the absolute deviation 
$|\RdelASn - \RiskASnP|$ in terms of the ``stability'' of $\Alg$, which is defined next.

\section{Notions of Stability for Learning Rules}
\label{sec:notions_of_stability}

In the following, we go briefly over some well-known notions of algorithmic stability, introduce the notion of $L_q$ 
stability coefficients and finally discuss its properties. 

The first known notion of \emph{algorithmic stability} is the so-called \emph{hypothesis stability}, or 
\emph{$L_1$-stability}, which is due to \citet{devroy_wagner_79B}.%
\footnote{The definitions are sometimes stated with their ``high probability'' variants. We prefer the expectation-versions as they fit our purposes better.}

\begin{definition}[Hypothesis Stability]
\label{def:L1_stability}
Algorithm $\Alg$ has hypothesis (or $L_1$) stability%
\footnote{We believe that in all these definitions the word ``sensitivity'' should be used rather ``stability''. To be in line with the 
literature, we kept the terminology, though with much doubt about whether this is the correct decision.}
 $\beta_h$ w.r.t to the loss function $\ell$ if the following holds
\begin{align*}
\forall i \in \curbra{1,\dots,n},~
\expectone{|\loss{\Alg(\sS_n),X}  -  \loss{\Alg(\sS_n^{-i}),X}|}
~ \leq ~
\beta_h ~,
\end{align*}
where randomness is over $\sS_n$ and $X \sim \fdistri{P}$, and $X$ is independent of $\fset{S}_n$.
\end{definition}

\citet{kearns_ron_neco_99} proposed a weaker notion of stability known as \emph{error stability} which measures the  absolute change in the expected loss of a learning algorithm instead of the average absolute pointwise change in the loss:
\begin{definition}[Error Stability]
\label{def:err_stability}
Algorithm $\Alg$ has error stability $\beta_e$
 w.r.t the loss function $\ell$ if the following holds
\begin{align*}
\forall i \in \curbra{1,\dots,n},~
| \expectone{\loss{\Alg(\sS_n),X}}  -  \expectone{\loss{\Alg(\sS_n^{-i}),X}}|
~ \leq ~
\beta_e ~,
\end{align*}
where randomness is over $\sS_n$ and $X \sim \fdistri{P}$, and $X$ is independent of $\fset{S}_n$.
\end{definition}
As noted by \citet{kutin_niyogi_02}, error stability is weaker than hypothesis stability (in the sense that if $\Alg$ has $\beta$ hypothesis stability then it also has $\beta$ error stability). Furthermore, this notion is not sufficiently strong to ``guarantee generalization'' in the sense that there are algorithms $\Alg$ such that their generalization gap, 
$\EE[ \Rres(\Alg,\sS_n) - R(\Alg(\sS_n),\fdistri{P})]$ stays positive, while the algorithm's error stability coefficient converges to zero as $n\to \infty$.

\cite{kale_kumar_cv_ms_stability_2011} proposed another weak notion of stability known as \emph{mean-square (MS) stability}, 
or \emph{$L_2$-stability}. 

\begin{definition}[Mean-Square Stability]
\label{def:L2_stability}
Algorithm $\Alg$ has mean-square (or $L_2$) stability $\beta_{ms}$ w.r.t the loss function $\ell$ if the following holds
\begin{align*}
\forall i \in \curbra{1,\dots,n},~
\expectone{(\loss{\Alg(\sS_n),X}  -  \loss{\Alg(\sS_n^{-i}),X})^2}
~ \leq ~
\beta_{ms} ~,
\end{align*}
where randomness is over $\sS_n$ and $X \sim \fdistri{P}$, and $X$ is independent of $\fset{S}_n$. 
\end{definition}
We comment on the relationship between mean-square stability and the other notions of stability shortly.%
\footnote{Note that in the above definition, $\beta_{ms}$ is not squared to keep our definitions in sync with the literature.}
\cite{bousquet_elisseeff_jmlr_2002} proposed the strongest and most strict notion of stability, known as 
\emph{uniform stability}, which implies all previous notions of stability.
This notion of stability, together with McDiarmid inequality, permitted the derivation of the first exponential
generalization error bound for the deleted estimate and the resubstitution estimate.
\begin{definition}[Uniform Stability]
\label{def:uni_stability}
Algorithm $\Alg$ has uniform stability $\beta_u$ w.r.t the loss function $\ell$ if the following holds
\begin{align*}
\forall~ \sS_n \in \fset{X}^n, ~\forall i \in \curbra{1,\dots,n},~\forall x\in \fset{X},~
| \loss{\Alg(\sS_n),x}  -  \loss{\Alg(\sS_n^{-i}),x} |
~ \leq ~
\beta_u ~.
\end{align*}
\end{definition}
Finally, we arrive at our notion of $L_q$ stability coefficients:
\begin{definition}[$L_q$ Stability Coefficient]
\label{def:Lq_stability}
Let $\sS_n$ be a sequence of $n$ \iid\ random variables (RVs) drawn from $\sX$ according to $\fdistri{P}$.
Let $\Alg$ be a deterministic learning rule, and $\ell$ be a loss function as defined in 
\cref{sec:setup_notations}.
For $q \ge 1$, the \emph{$L_q$ stability} coefficient of $\Alg$ w.r.t $\ell$, $\fdistri{P}$, and $n$ is denoted 
by $\beta_q(\Alg, \ell, \fdistri{P}, n)$ and is defined as
\begin{align*}
\beta_q^2 (\Alg, \ell, \fdistri{P}, n)
& ~ = ~
\frac1n \sum_{i=1}^n \norm{ \loss{\Alg(\sS_n),X} - \loss{\Alg(\sS_n^{-i}),X}}_q^2\,,
\end{align*}
where $X \sim \fdistri{P}$ is independent of $\fset{S}_n$.
\end{definition}
Recall that a learning algorithm is symmetric, if 
$\Alg(\sS_n) = \Alg(\sS_n')$ for any two $\sS_n$, $\sS_n'$ which are reorderings of each other.
For symmetric learning algorithms, the above definition simplifies to 
\begin{align*}
\beta_q(\Alg,\ell,\fdistri{P},n) = 
 \norm{ \loss{\Alg(\sS_n),X} - \loss{\Alg(\sS_n^{-1}),X}}_q
 =
\max_i  
 \norm{ \loss{\Alg(\sS_n),X} - \loss{\Alg(\sS_n^{-i}),X}}_q\,,
\end{align*}
 the latter expression coinciding with the definition given by \citet{celisse_guedj_2016}.
Thus, for a symmetric algorithm, the stability coefficient 
$\beta_q(n)\doteq \beta_q(\Alg,\ell,\fdistri{P},n)$ is in fact a $q$--norm for the RV
$\Delta_n(\Alg) := \loss{\Alg(\sS_n),X} - \loss{\Alg(\sS_n^{-1}),X}$.
The reason we chose the particular averaging in our definition is because this definition gives the best fit to the derivations we use.

We note in passing that the other stability notions could also consider averaging over index $i$, instead of taking the worst-case sensitivity as shown above. For symmetric rules, the difference, again, does not matter. However, for nonsymmetric algorithms this difference is nontrivial: While uniform stability is trivial for binary classification with the current definition, it is a useful notion with averaging as shown by the results of \citet{shwartz_shamir_jmlr2010}.

Since often $\Alg$, $\ell$, $\fdistri{P}$, $n$ are fixed, we will drop them (or any of them) from the notation and 
will just use, for example, $\beta_q, \beta_q(n)$, etc.%
\footnote{This should not be mistaken to taking a supremum over any subset of the dropped quantities: The stability 
coefficients are meant to be algorithm, loss and distribution dependent.}
Note that for $1\le q\le q'$, it holds that $\beta_q \le \beta_{q'}$, which follows from the definition of $q$-norms.
Now, the relationship between the various stability concepts becomes clear.
Taking $\beta_e,\beta_h,\beta_{ms},\beta_u$ as the smallest values that are possible (i.e., changing the inequalities in 
their definitions to equalities), assuming symmetric learning rules, we have
$\beta_e \le \beta_1 = \beta_h \le \beta_2 =\beta_{ms}^{1/2} \le \beta_u$,
where the last inequality follows because the $L^\infty$-norm is the largest of all of the $q$-norms.

The $L_q$ stability coefficient quantifies the variation of the loss of $\Alg$ induced by removing one example from the 
training set.
This is known as a \emph{removal type} notion of stability and is in accordance with the previous notions of stability 
introduced earlier.
The difference between $L_q$ stability and earlier notions of stability is that $L_q$ stability is in terms of the higher 
order moments of the RVs $|\loss{\Alg(\sS_n),X} - \loss{\Alg(\sS_n^{-i}),X}|$.
The reason we care about higher moments is because we are interested in controlling the tail behavior of the deleted 
estimate.
As will be shown, the tail behavior of the deleted estimate is also dependent on the tail behavior of RVs characterizing 
stability.
As is well-known, knowledge of the higher moments of a RV is equivalent to knowledge of the tail behavior of the RV.
As such, controlling the higher order moments provides more information on the distribution of this RV than simply
considering first order ($L_1$) and second order ($L_2$) moments.
As it will turn out, the $L_q$ stability coefficients alone are insufficient to control either the bias, or the the tail behavior of the resubstitution estimate. To control these, we will introduce further stability coefficients, but we prefer to do this just before we need them.

\section{Main Results}
\label{sec:main_result}

We give here the main results of our work, namely an exponential tail bound for the concentration of the estimated risk, 
expressed in terms of the deleted, or the resubstitution estimate. We start with the deleted estimate.

Before stating the result for the deleted estimate, we first state our two assumptions,
both of which concern the behavior of the stability coefficients. 
While the first assumption is concerned with their dependence on $n$, the second assumption is concerned with their behavior as a function of $q$.
\begin{assumption}
\label{assump:beta_vs_n}
For a fixed $q > 0$, $\beta_q(n)$ is a nonincreasing function of $n$.
\end{assumption} 
Now note that our results remain valid if $\beta_q(n)$ is replaced with an upper bound on it (such as $\bar \beta_q(n)$), provided that the upper bound satisfies our assumptions.
Defining $\bar\beta_q(n) = \max_{m\ge n} \beta_q(n)$, we find that the map $n \mapsto \bar \beta_q(n)$ is nonincreasing. 
This provides us with a general approach to meet \cref{assump:beta_vs_n}, although we would often expect this assumption to be met anyways.

\begin{assumption}
\label{assump:one}
$\exists$ $u_1,w_1 \ge 0$ s.t. for any integer $q \ge 1$, it holds that
\begin{align}
2n\beta_{4q}^2(n-1) + \sfrac{2}{n^2} \sum_{i=1}^n \norm{\loss{ \Alg(\sS_n^{-i} ) , X_i}}_{4q}^2  \le \sqrt{qu_1} \vee qw_1 ~,
\label{eq:assumpone}
\end{align}
where $a \vee b = \max(a,b)$.
\end{assumption}
This assumption will be clarified once we introduce our main tool (the exponential Efron-Stein inequality) and the notion 
of sub-gamma random variables in the following sections. 
The reader wondering about whether this assumption can be met will be pleased to find a positive answer presented in  \cref{subsec:example_ridgeregression} for the case of unbounded response ridge regression, where the assumption translates into conditions for the tail behavior of the response variable.
Note that here $u_1,v_1$ will be distribution and sample size dependent constants, generally decreasing with the sample size.

With this, our main result for the deleted estimate is as follows:
\begin{restatable}[Deleted estimate tail bound]{theorem}{theoremfinalgenresult}
\label{theorem:exp_tail_bound_for_del}
Let $\sX$, $\sH$ and $\ell$ be as previously defined, $\ell$ bounded in $[0,1]$.
Let $\sS_n$ be the dataset defined in \cref{subsec:risk_estimates}, where $n \geq 2$.
Let $\Rdel\bra{\Alg,\sS_n}$ be the deleted estimate defined in \cref{eq:del_estimate}, and $R(\Alg(\sS_n),\fdistri{P})$ 
be the risk for hypothesis $\Alg(\sS_n)$.
Then, under \cref{assump:beta_vs_n,assump:one}, for $\delta \in (0,1)$ and $a > 0$, with probability $1 - 2\delta$ 
the following holds
\begin{align}
|\Rdel\bra{\Alg,\sS_n} - R(\Alg(\sS_n),\fdistri{P})|
\le
\beta_1(n)
+ 
4 \sqrt{(n\beta_2^2(n-1) + C_1)\logtwodelta}
+
C_2\logtwodelta ~,
\end{align}
where 
$C_1 = 2.2a^2u_1 + 1.07a^2w_1^2$, and $C_2 =\sfrac{4}{3}(1.46aw_1 + \sfrac{1}{a})$.
Further, when the range of $\ell$ is unrestricted ($\ell$ is unbounded), the same inequality holds under the assumption that a modified version of \cref{assump:one} holds where the LHS of \cref{eq:assumpone} multiplied by the constant $4$.
\end{restatable}
While the above result bounds both sides of the tail, a one side version with $\logtwodelta$ replaced by $\logonedelta$, also holds for both the upper and lower tails. We will soon explain the various terms in this bound, but first let us explain how the result is proven.

Once we establish our exponential Efron-Stein inequality (\cref{lemma:exp_efron_stein}), the proof of \cref{theorem:exp_tail_bound_for_del} is relatively straightforward. %
The essence of the proof 
can be summarized as follows:
In order to control the concentration of the random quantity $\Rdel\bra{\Alg,\sS_n}$ around the true risk, 
we study the concentration of $\Rdel\bra{\Alg,\sS_n}$ around its mean, the concentration of the true risk of $\Alg(\sS_n)$ around its own mean, and the difference between the mentioned means. The latter is bounded by the $\beta_1(n)$ stability coefficient, using elementary arguments. To 
control the tails (or the higher order moments) of $\Rdel\bra{\Alg,\sS_n}$,
we use our exponential Efron-Stein inequality, which tells us that we need 
to control the tails of 
another intermediary random variable, $\Vdel$ (see \cref{cor:efron_stein_remove}
and its use in \cref{subsec:bounding_term_I}), a variance-type
measure of the sensitivity of $\Rdel$ to the removal of one of the training examples at a time.
The moments of $\Vdel$ are shown be controlled by the expression on the LHS of \cref{eq:assumpone} of \cref{assump:one}. The assumption then helps to turn these bounds into a bound on the MGF of $\Vdel$, leading to tail bounds.
The concentration of the true risk of $\Alg(\sS_n)$ around its mean is controlled similarly.
The full proof is presented in \cref{sec:exptailbound_for_del_estimate}.

To interpret the bound, it is worthwhile to simplify it at the price of losing a bit on its tightness. Consider the case when the range of $\ell$ is the $[0,1]$ interval. Then, 
further upper bounding the RHS using $\sqrt{x+y}\le \sqrt{x}+\sqrt{y}$ and then choosing $a$ optimally, yields the following simplified bound
\begin{align*}
|\Rdel\bra{\Alg,\sS_n} - R(\Alg(\sS_n),\fdistri{P})| 
& \le
\beta_1(n) + 4  \sqrt{n\beta_2^2(n-1)\logtwodelta} +\\
&  
+ 8
\sqrt{ 
 \sfrac{1}{3}
 \bra{   \sqrt{ (2.2 u_1 + 1.07 w_1^2)} + 
\sfrac{1}{3} 1.46 w_1 }
 } \logtwodelta \,,
 \numberthis
 \label{eq:simplifiedbound}
\end{align*}
where for the final form, we assumed that $\delta \le 1/e$.
What can be noticed is that the tail bound has the form that we expect to see for sub-gamma RVs; note the presence of 
the $\sqrt{\logtwodeltaf}$ and $\logtwodeltaf$ terms. 
Note also that, by assumption, $u_1^{1/2}$ and $w_1$ are both at least of size $\Omega(1/n)$, 
regardless the stability of $\Alg$. 
As a result, the coefficient of the $\logtwodeltaf$ term is at least of order $\Omega(\sqrt{1/n})$ even for algorithms where $\beta_q = 0$.
This is expected because the deleted estimate is a ``noisy'' estimate of the true risk no matter the algorithm --
one expects an $\Omega(\sqrt{1/n})$ lower bound to hold in general.
Finally, we note in passing that with a bit more care, in the case of $w_1=0$ 
it is possible to slightly reduce the exponent of $\logtwodeltaf$ to $\log^{3/4}(2/\delta)$. 

We can gain further insight by qualitatively comparing our bound in \cref{theorem:exp_tail_bound_for_del} with the 
exponential bound for the deleted estimate obtained by \citet[Theorem 12]{bousquet_elisseeff_jmlr_2002}.
To make the comparison easier, we first state their result using our notation. We also give the two sided version.

\begin{theorem}[Deleted estimate tail bound through uniform stability]
\label{theorem:bne_expgenbound_del_reg}
Let $\Alg$ be a learning rule with uniform stability $\beta_u$ (see \cref{sec:notions_of_stability}) with respect 
to the loss function $\ell$ and assume that this loss function is in addition bounded:
$0 \le \loss{\Alg(\sS_n),X}\le M$ holds almost surely.
Then, for any $n \ge 1$, and any $\delta \in (0,1)$, with probability $1 - \delta$, the following holds
\begin{align*}
\abs{R(\Alg(\sS_n),\fdistri{P}) - \Rdel\bra{\Alg,\sS_n}}
\le
\beta_u(n) + 4n\beta_u(n)\sqrt{\sfrac{\log(2/\delta)}{2n}} + M\sqrt{\sfrac{\log(2/\delta)}{2n}} ~.
\numberthis
\label{eq:bebound}
\end{align*}
\end{theorem}

The bound in \cref{theorem:bne_expgenbound_del_reg} has three terms; 
the first two terms are dependent on the uniform stability of the learning rule, and a third term that only depends on 
the loss function $\ell$ and the sample size $n$.
When $\beta_u$ scales as $1/n$ the bound becomes tight in a worst-case sense. 
As with our bound, even when $\beta_u=0$, the third term stays positive (as it should).

Let us now compare the RHS of \cref{eq:bebound} to our simplified bound presented in \cref{eq:simplifiedbound}. 
Both bounds have three terms, corresponding to different power of $\logonedelta$.
The first two terms in \cref{eq:bebound} have the same form as the first two terms in \cref{eq:simplifiedbound} except 
that in \eqref{eq:bebound} $\beta_u$ is used, while in \eqref{eq:simplifiedbound} $\beta_1(\le \beta_2)$ is used. 
As discussed earlier, $\beta_1\le \beta_2\le \beta_u$, and in particular the gap between these quantities can be large; even $\beta_u$ may be unbounded while the others bounded (as the example in  \cref{subsec:example_ridgeregression} will show).
At the same time, the constant coefficient of the second term \eqref{eq:simplifiedbound} is larger than the corresponding 
coefficient in the second term of \eqref{eq:bebound}. 
Other than these differences, the terms are analogous.

As discussed earlier, our last term scales with $\logonedelta$ rather than with with its square root. 
This is the price we pay partly because the proof is set up to also work with \emph{unbounded losses} (it remains an interesting question of whether the result can be strengthened to remove this term when the losses are bounded). Note that the coefficients of this term, as 
discussed earlier, depend on the stability of the algorithm 
but the magnitude of the multiplier of $\logonedelta$ 
will be at least of order $\Omega(\sqrt{1/n})$.

Note that \cite{bousquet_elisseeff_jmlr_2002} state an identical result (to that shown above in \eqref{eq:bebound}) for 
the generalization gap, $\Rres - R(\Alg(\sS_n),\fdistri{P})$; i.e. the gap for the resubstitution
estimate (or the training error). As we shall see soon, our bound also extends to this case with some modifications.

While preparing the final version of this manuscript, we noted the recent work of \cite{feldman_vonfdrak_nips2018} who 
improve this result of \cite{bousquet_elisseeff_jmlr_2002} by replacing the second term in 
\eqref{eq:bebound} by $\sqrt{\beta_u(n) \logtwodelta}$. 
This is an improvement whenever $\beta_u(n)\ge 1/n$ (i.e., for ``not too stable'' algorithms).
One may hope that a similar improvement may be possible with non-uniform (distribution-dependent) notions of stability, but this is left for future work for now.

Notice that the above results are for the gap between the true risk and the deleted estimate, whereas oftentimes one wishes 
to control the gap between the true risk and the resubstitution (or \emph{empirical}) estimate (or the training error); 
i.e., the well-known \emph{generalization gap}.
Indeed, one can follow the same path for our proof technique and derive an exponential 
tail bound for the concentration of the empirical estimate.

As it turns out, this result requires the introduction of a new type of stability coefficients: 
The reason is that there are stable algorithms that can overfit the training data in the sense that their training error is small. An example of such an algorithm for the binary classification setting is the ``short-range nearest neighbor'' rule which recalls the label of the closest training example to the input when their distance is $o(1/n)$ and outputs a fixed label (say, $1$) otherwise. As $n$ increases, this algorithm will converge to output the a priori chosen label always. As such, the algorithm will also be very stable. Yet its training error is always zero, which can be far from its true risk. The situation is summarized in the following result (for details, see \cref{appx:example_short_range_1nn}):

\begin{restatable}{proposition}{propsrknnlqstability}
\label{prop:srknn_lq_stability}
There exist a distribution $\cP$ and a learning algorithm $\Alg$ such that, everywhere,
\begin{align}
\lim_{n\to\infty} R(\Alg(\sS_n),\cP) - 
\Rres\bra{\Alg,\fset{S}_n} >0~,
\label{eq:posgap}
\end{align}
while $\sup_{q\ge 1}\beta_q(\Alg,\ell,\cP,n)/q \to 0$ as $n\to\infty$.
\end{restatable}

Note that by \cref{theorem:exp_tail_bound_for_del}, $\Rdel\bra{\Alg,\fset{S}_n} - R\bra{\Alg(\sS_n),\cP} \stackrel{P}{\to} 0$ as long as $\sup_{q\ge 1} \frac{\beta_q(n)}{q} \to 0$ as $n \to \infty$.
It follows that the deleted estimate is consistently estimating the risk of the short-range nearest neighbor rule, while the resubstitution estimate fails to be consistent. While it is common wisdom that the resubstitution estimate is often overly ``optimistic'', the example is a very clear demonstration of this weakness and shows that one has to be quite careful when using the training error, e.g., for model selection; as noted already in \citet{DeWa79resub}.

One may then think that we should never use the training error, but this is easier said than done for most algorithms will in some form minimize the training error. Thus, the question remains, if the $L_q$ stability in the previous sense is insufficient to guarantee the concentration of the training error around the true loss, what other property should an algorithm posses to control this concentration? 
We know that uniform stability provides a positive answer,
but is there an analogue to the $L_q$ stability coefficients that is sufficient for this purpose? The answer to this question can be obtained by repeating the derivations done in the proof of \cref{theorem:exp_tail_bound_for_del} and discovering the modifications necessary to control all the terms. This results in the definition of what we call the $L_q$ resubstitution stability coefficients, which, given an algorithm $\Alg$, are defined as follows:
\begin{align*}
\gamma_{q}^2(n) = \frac{1}{n} \sum_{i=1}^n
 \norm{\loss{ \Alg(\sS_n ) , X_i}  -  \loss{\Alg(\sS_n^{-i}), X_i} }_{q}^2\,.
\end{align*}
This is a direct analogue of the $L_q$ stability coefficients: The main difference is that here, the algorithm is evaluated on training examples, with and without the example being removed, while for the $L_q$ stability coefficients, the algorithm was evaluated on an example independent of the training data.
This should make sense for already if we want to control the bias $\E[ \Rres(\Alg,\sS_n) - R(\Alg(\sS_n),\cP) ]$ we see the need to control the deviation between the loss measured at training examples and the loss measured outside -- which is exactly what is captured by the $\gamma_q$ coefficients.

We also replace \cref{assump:one} with the following assumption:
\begin{assumption}
\label{assump:subgammagamma}
$\exists$ $u_1,w_1 \ge 0$ s.t. for any integer $q \ge 1$, it holds that
\begin{align*}
  6n 
  \bra{\gamma_{4q}^2(n) + \gamma_{4q}^2(n-1) +  \beta_{4q}^2(n-1)}
  +
 \frac{2 \gamma_{4q}^2(n)}{n} 
 +
  \frac{2}{n^2}
 \sum_{i=1}^n \norm{\loss{ \Alg(\sS_n^{-i}) , X_i}}_{4q}^2
  \le \sqrt{qu_1} \vee qw_1 ~.
\end{align*}
\end{assumption}
Note that this assumption implies \cref{assump:one}. Thus, any algorithm that satisfies this assumption, will also necessarily satisfy  \cref{assump:one}. 
With this, we are ready to state our result for tail of the resubstitution estimator:
\begin{restatable}[Resubstitution estimate tail bound]{theorem}{theoremfinalgenresultres}
\label{theorem:exp_tail_bound_for_res}
Using the setup of \cref{theorem:exp_tail_bound_for_del} (again, $\ell\in [0,1]$), but using 
\cref{assump:subgammagamma} in place of \cref{assump:one},  
 for $\delta \in (0,1)$ and $a > 0$, with probability $1 - 2\delta$ 
the following holds:
\begin{align*}
\MoveEqLeft |\Rres\bra{\Alg,\sS_n} - R(\Alg(\sS_n),\fdistri{P})|
\le
\beta_1(n) + \gamma_1(n)+\\
& \qquad 
4 \sqrt{(n\beta_2^2(n-1) + C_1)\logtwodelta}
+
C_2\logtwodelta ~,
\end{align*}
where  $C_1=C_1(a)$ and $C_2=C_2(a)$ are as in \cref{theorem:exp_tail_bound_for_del}.
Furthermore, the same holds for unbounded losses provided
that a modified version of \cref{assump:subgammagamma} holds 
where the LHS of the inequality in this assumption is  multiplied by the constant $4$.
\end{restatable}
By and large, the proof follows the same line as the proof \cref{theorem:exp_tail_bound_for_del} with some necessary modifications. 
A proof  of this result is provided in \cref{sec:resproof}.
As can be noticed, the only difference to the bound available for the deleted estimate is the presence of the $(\gamma_q(n))$ terms, both in the assumption, and the result.
As our previous example shows, these cannot be removed (in the case of the short-range nearest neighbor rule, these coefficients will be large).
This suggests that the \emph{deleted estimate} is in a way a \emph{much better} behaving estimator of the true risk than 
the resubstitution estimate.

\section{Main Tool}
\label{sec:main_tool}

In this section we focus on the case when the losses are bounded in $[0,1]$ and comment on the general case at the end.
The main tool for our work is an extension of the celebrated Efron-Stein inequality \citep{efron_stein_1981,steele_1986}, 
to a stronger version known as the exponential Efron-Stein inequality \citep{boucheron_lugosi_massart_2003}. 
We start by introducing the Efron-Stein inequality and some variations.
Let $\map{f} {\sX^n} {\RR}$ be a real-valued function of $n$ variables, where $\sX$ is a measurable space.
Let $X_1,\dots,X_n$ be independent (not necessarily identically distributed) RVs taking values in $\sX$ and define the 
RV $Z ~ = ~ f(X_1,\dots,X_n) \equiv f(\fset{S}_n)$.
Define the shorthand for the conditional expectation $\EE_{-i}Z \doteq \expectone{ Z | \sS_n^{-i} }$, where 
$\sS_n^{-i}$ is defined as in the previous section.
Informally, $\EE_{-i}Z$ ``integrates'' $Z$ over $X_i$ and \emph{also over any other source of randomness} in $Z$ except 
$\sS_n^{-i}$.
For every $i = 1,\dots,n$, let $X_i'$ be an independent copy from $X_i$, and let 
$Z_i' = f(X_1,\dots,X_{i-1},X_i',X_{i+1},\dots,X_n)$.
The Efron-Stein inequality bounds the variance of $Z$ as shown in the following theorem.

\begin{theorem}[Efron-Stein Inequality -- Replacement Case]
\label{theorem:efron_stein_main}
Let $V = \sum_{i=1}^n (Z - \EE_{-i}Z)^2$.
Under the settings described in this section, it holds that
$\VV[Z] \leq \EE V  = \sfrac{1}{2} \sum_{i=1}^n \EE[(Z - Z_i')^2]$.
\end{theorem}

The proof of \cref{theorem:efron_stein_main} can be found in \citep{boucheron_lugosi_massart_2004}.
Another variant of the Efron-Stein inequality that is more useful for our context, is concerned with the removal of one 
example from $\sS_n$.
To state the result, let $\map{f_i} {\sX^{n-1}} {\RR}$, for $1 \le i \le n$, be an arbitrary measurable function, and 
define the RV $Z_{-i} = f_i(\fset{S}_n^{-i})$.
Then, the Efron-Stein inequality can be also stated in the following interesting form \citep{boucheron_lugosi_massart_2004}.

\begin{restatable}[Efron-Stein Inequality -- Removal Case]{corollary}{corefronsteinremove}
\label{cor:efron_stein_remove}
Assume that $ \expectop_{-i}[Z_{-i}]$ exists for all $1\le i \le n$, and let $\Vdel = \sum_{i=1}^n { (Z - Z_{-i})^2}$.
Then it holds that
\begin{equation}
\VV[Z]
~ \le ~
\expectop V
~ \le ~
\expectop \Vdel
~.
\end{equation}
\end{restatable}
The proof of \cref{cor:efron_stein_remove} is given in \cref{appx:cor_efron_stein_remove}.
This is a standard proof and it is replicated for the reader's benefit only. 

\subsection{An Exponential Efron-Stein Inequality}
\label{subsec:exp_efron_stein}

The work of \citet{boucheron_lugosi_massart_2003} has focused on controlling the tail of general functions of independent 
RVs in terms of the tail behavior of Efron-Stein variance-like terms such as $V$ and $\Vdel$, as well as other terms known 
as $V^+$ and $V^-$.
The variance-like terms $V$, $V^+$ and $V^-$ measure the sensitivity of a function of $n$ independent RVs w.r.t the
\emph{replacement} of one RV from the $n$ independent RVs.
The term $\Vdel$ on the other hand, measures the sensitivity of a function of $n$ independent RVs w.r.t the 
\emph{removal} of one RV from the $n$ independent RVs.
In this work, we favor $\Vdel$ over the other terms since it is more suitable for our choice of stability coefficient
(the $L_q$ stability), which is also a removal version. 
The removal version of stability is preferred as it is more natural in the learning context where one is given a fixed 
sample. 
an interesting future direction), where working with the replacement version will need extra data, or extra assumptions.
It also leads to simpler/shorter calculations, and saves a (small) constant factor in the bounds. The problem is that it can only be applied in the case of bounded losses. 

The tail of a RV is often controlled through bounding the logarithm of the moment generating function (MGF) of the RV.
This is known as the \emph{cumulant generating function} (CGF) of the RV and is defined as
\begin{align}
\psi_Z(\lambda) \doteq \log \expectone{ \exp( \lambda Z ) },
\label{eq:cgf_rv_z}
\end{align} 
where $\lambda\in \mathrm{dom}(\psi_Z)\subset \mathbb{R}$ and belongs to a suitable neighborhood of zero.
The main result of \cite{boucheron_lugosi_massart_2003} bounds $\psi_Z$ in terms of the MGF for $V$, $V^+$ and $V^-$, 
but not in terms of the MGF for $\Vdel$.
Since we are particularly interested in the RV $\Vdel$, the following theorem bounds the tail of $\psi_Z$ in terms of
the MGF for $\Vdel$.

\begin{restatable}{theorem}{thmcgfzvdelexpbound}
\label{theorem:cgf_z_vdel_expbound}
Let $Z$, $\Vdel$ be defined as in \cref{cor:efron_stein_remove} and assume that $|Z-Z_{-i}|\le 1$ almost surely 
for all $i$.
For all $\theta > 0$, s.t. $\lambda \in (0,1]$, 
$\theta\lambda < 1$, and $\expectop e^{\lambda \Vdel} < \infty$,
the following holds
\begin{equation}
\log\expectone{\exp\bra{ \lambda (Z - \expectop Z) }}
\le
\sfrac{\lambda\theta}{(1 - \lambda \theta)}
\log\expectone{\exp\bra{\sfrac{\lambda\Vdel}{\theta}}}. \label{eq:cgf_z_vdel}
\end{equation}
\end{restatable}

The proof of \cref{theorem:cgf_z_vdel_expbound} is given in \cref{appx:theorem_cgf_z_vdel_expbound}.
\cref{theorem:cgf_z_vdel_expbound} states that the CGF of the centered RV $Z - \EE Z$ is upper bounded by the CGF 
of the RV $\Vdel$.
Hence, when $\Vdel$ behaves ``nicely'', the (upper) tail of $Z$ can be controlled.
The value of $\theta$ in the upper bound is a free parameter that can be optimized to give the tightest bound.
Because $\lambda>0$, the bound is \cref{eq:cgf_z_vdel} is only for the upper tail of the RV $Z$.
A similar bound for the lower tail can be obtained by replacing $Z$ with $-Z$ and applying the result.
Note also that for both sides, upper tail and lower tail, the same requirements for $\lambda$ and $\theta$ in
\cref{theorem:cgf_z_vdel_expbound} apply.

For \cref{theorem:cgf_z_vdel_expbound} to be useful in our context, further control is required to upper bound
the tail of the RV $\Vdel$.
Our approach to control the tail of $\Vdel$ will be, again, through its CGF.
In particular, we aim to show that when $\Vdel$ is a sub-gamma RV (defined shortly) we can obtain a high probability 
tail bound on the deviation of the RV $Z$.
The obtained tail bound will be instrumental in deriving the exponential tail bound for the deleted estimate.

\subsection{Sub-Gamma Random Variables}
\label{subsec:subgamma_rvs}

We follow here the notation of \cite{book_concentration_2013}.
A real valued centered RV $X$ is said to be \emph{sub-gamma} on the right tail with variance factor $v$ and scale 
parameter $c$ if for every $\lambda$ such that $0 < \lambda < 1/c$, the following holds
\begin{eqnarray}
\psi_X(\lambda) & \le & \frac{\lambda^2 v}{2(1 - c\lambda)}~. \label{kooky}
\end{eqnarray}
This is denoted by $X \in \Gamma_+(v,c)$.
Similarly, $X$ is said to be a sub-gamma RV on the left tail with variance factor $v$ and scale parameter $c$ if
$-X \in \Gamma_+(v,c)$.
This is denoted as $X \in \Gamma_-(v,c)$.
Finally, $X$ is simply a sub-gamma RV with variance factor $v$ and scale parameter $c$ if both $X \in \Gamma_+(v,c)$ and 
$X \in \Gamma_-(v,c)$.
This is denoted by $X \in \Gamma(v,c)$.

The sub-gamma property can be characterized in terms of moments conditions or tail bounds.
In particular, if a centered RV $X \in \Gamma(v,c)$, then for every $t > 0$,
\begin{align}
\probone{X > \sqrt{2vt} + ct} \vee \probone{-X > \sqrt{2vt} + ct} 
& 
\le e^{-t}~,
\end{align}
where $a \vee b = \max(a,b)$.
The following theorem from \citep{book_concentration_2013} characterizes this notion more precisely:

\begin{theorem} 
\label{theorem:moments_charac_subgamma}
Let $X$ be a centered RV. If for some $v > 0$ and $c \ge 0$
\begin{align}
\probone{X > \sqrt{2vt} + ct}
\vee
\probone{-X > \sqrt{2vt} + ct}
\le
e^{-t}
~,
~ \text{for every} ~t > 0 ~,
\label{eq:tail_bound_subgamma_rv}
\end{align}
then for every integer $q \ge 1$
\begin{align*}
\norm{X}_{2q}
~ \le ~
(q! A^q + (2q)! B^{2q})^{1/{2q}}
~ \le ~
\sqrt{16.8 qv} \vee 9.6 qc ~
~ \le ~
10 (\sqrt{qv} \vee qc )\,,
\end{align*}
where $A=8v$, $B=4c$.
Conversely, if for some positive constants $u$ and $w$, for any integer $q \ge 1$,
\begin{align*}
\norm{X}_{2q}
& ~ \le ~
\sqrt{qu} \vee qw ~,
\end{align*}
then $X \in \Gamma(v,c)$ with $v=4(1.1 u + 0.53w^2)$ and $c=1.46w$, 
and therefore \eqref{eq:tail_bound_subgamma_rv} also holds.
\end{theorem}

The reader may notice that \cref{theorem:moments_charac_subgamma} is slightly different than the version in the book of 
\citet{book_concentration_2013}.
Our extension is based on simple (and standard) calculations that are merely for convenience with respect to our purpose.

\subsection{An Exponential Tail Bound for $Z$}

In this section we assume that the centered RV $\Vdel-\EE\Vdel \in \Gamma(v,c)$ with variance factor $v > 0$, scale 
parameter $c \ge 0$,
$0<c|\lambda|<1$.
Hence, from inequality \eqref{kooky} it holds that
\begin{align*}
\psi_{ \Vdel - \EE{\Vdel} }(\lambda) 
& =
\log \expectone{ \exp( \lambda (\Vdel - \expectop\Vdel) ) }
 \le
\sfrac{1}{2}\lambda^2 v(1 - c|\lambda|)^{-1}~.
\end{align*}
The sub-gamma property of $\Vdel$ provides the desired control on its tail.
That is, after arranging the terms of the above inequality, the CGF of $\Vdel$ which controls the tail of $\Vdel$, is
upper bounded by the deterministic quantities: $\EE\Vdel$, the variance $v$, and the scale parameter $c$.

It is possible now to use the sub-gamma property of $\Vdel$ in the result of the exponential Efron-Stein inequality in 
\cref{theorem:cgf_z_vdel_expbound}.
In particular, the following lemma gives an exponential tail bound on the deviation of a function of independent RVs,
i.e. $Z = f(X_1,\dots,X_n)$, in terms of $\EE\Vdel$, the variance factor $v$, and the scale parameter $c$.
This lemma will be our main tool to derive the exponential tail bound on the DEL estimate.

\begin{restatable}{lemma}{lemmaexpefronstein}
\label{lemma:exp_efron_stein}
Let $Z$, $Z_{-i}$, $\Vdel$ be as in \cref{cor:efron_stein_remove}.
If $\Vdel-\EE\Vdel$ is a sub-gamma RV with variance parameter $v > 0$ and scale parameter $c \ge 0$, then for any 
$\delta \in (0,1)$, $a>0$, with probability $1 - \delta$, 
\begin{equation}
\abs{Z - \EE Z}
\le
\sfrac{2}{3}(ac+1/a)\logtwodelta
+
2\sqrt{(\EE\Vdel + a^2v/2) \logtwodelta}.
\end{equation}
\end{restatable}

The proof of \cref{lemma:exp_efron_stein} is given in \cref{appx:lemma_exp_efron_stein}.
Parameter $a$ is a free parameter that can be optimized to give the tightest possible bound.
In particular, $a$ can be chosen to provide the appropriate scaling for the RV $Z$ such that the bound goes to zero as
fast as possible.
A typical choice of $a$ would be the inverse standard deviation of $Z$.

\section{Proof of \cref{theorem:exp_tail_bound_for_del}}
\label{sec:exptailbound_for_del_estimate}

In this section we derive our main result given in \cref{theorem:exp_tail_bound_for_del}, namely the concentration
of the following random quantity
\begin{equation*}
| \Rdel(\Alg,\sS_n) - R(\Alg(\sS_n),\fdistri{P}) | ~.
\end{equation*}
To bound this RV, we decompose it into three terms
\begin{equation}
\abs{\Rdel(\Alg,\sS_n) - R(\Alg(\sS_n),\fdistri{P})}
~ \le ~
\mathrm{I} + \mathrm{II} + \mathrm{III}~,
\label{ineq:rdel_deviation_3mainterms}
\end{equation}
where 
\begin{align*}
\mathrm{I}   & =  |\EE\Rdel(\Alg,\sS_n) - \Rdel(\Alg,\sS_n)| ~, \\
\mathrm{II}  & =  |\EE R(\Alg(\sS_n),\fdistri{P}) - R(\Alg(\sS_n),\fdistri{P})| ~, \quad \text{and} \\
\mathrm{III} & =  |\EE R(\Alg(\sS_n),\fdistri{P}) - \EE\Rdel(\Alg,\sS_n)| ~.
\end{align*}
If the three terms in the RHS of \eqref{ineq:rdel_deviation_3mainterms} are properly upper bounded, we will have
the desired final high probability bound.
Terms I and II shall be bounded using the exponential Efron-Stein inequality in \cref{lemma:exp_efron_stein}.
Further, we hope that the final upper bounds can be in terms of the $L_q$ stability coefficient of $\Alg$.
Term III, however, is non-random and thus shall be directly bounded using some $L_q$ stability coefficient.

For terms $\text{I}$ and $\text{II}$, the key quantity for using the exponential Efron-Stein inequality in 
\cref{lemma:exp_efron_stein} is the RV $\Vdel$.
In particular, the requirement for using $\Vdel$ is two-fold.
First, since $\Vdel = \sum_{i=1}^n (Z - Z_{-i})^2$, where $Z_{-i} = f_i(\sS_n^{-i})$ for some function $f_i$, we need to 
choose $f_i$ appropriately.
Second, once $Z_{-i}$ is defined, to be able to use \cref{lemma:exp_efron_stein} we need to show that $\Vdel$ is a 
sub-gamma RV. 
For this, using \cref{theorem:moments_charac_subgamma}, it will be sufficient to show that for all integers $q \ge 1$,
\begin{align}
\norm{\Vdel}_{2q}
\leq
\sqrt{qu} \vee qw~,
\label{eq:vdelsg}
\end{align}
for some positive constants $u$ and $w$.
Here, we will relate $\norm{\Vdel}_{2q} $ to $L^q$ stability coefficients and then we ``reverse engineer'' appropriate 
assumptions on the $L^q$-stability coefficients that imply \eqref{eq:vdelsg}.

\subsection{Upper Bounding Term $\mathrm{I}$}
\label{subsec:bounding_term_I}

We begin by deriving an upper bound for term I in the RHS of $\eqref{ineq:rdel_deviation_3mainterms}$.
This is the deviation $|\EE\Rdel(\Alg,\sS_n) - \Rdel(\Alg,\sS_n)|$.
Note that $\Rdel(\Alg,\sS_n)$ is a function of $n$ independent random variables (bounded in $[0,1]$ when $\ell$ is in this range), and hence the Exponential Efron-Stein
inequality from \cref{lemma:exp_efron_stein} can be applied to bound this deviation.
Following our two-steps plan to use \cref{lemma:exp_efron_stein}, we define the random variables $Z$ and $Z_{-i}$
as follows
\begin{equation}
Z      ~ = ~ \Rdel(\Alg,\sS_n) = \frac{1}{n}\sum_{i=1}^n \loss{ \Alg(\sS_n^{-i} ) , X_i}, \qquad
Z_{-i} ~ = ~ \frac{1}{n} \sum_{\overset{j=1}{\sss j \neq i}}^n \loss{\Alg(\sS_n^{-i , -j}), X_j},
\label{eq:Z_and_Zi_for_Rdel}
\end{equation}
where $\sS_n^{-i , -j}$ indicates the removal of examples $X_i$ and $X_j$ from $\sS_n$.
Note that $Z_{-i} = \sfrac{n-1}{n} \Rdel(\Alg,\sS_n^{-i})$ -- the scaling factor is chosen to minimize the bound soon to be presented.
Recall that $\Vdel = \sum_i (Z - Z_{-i})^2$, and given the definition of $Z$ and $Z_{-i}$ in 
\eqref{eq:Z_and_Zi_for_Rdel}, we need to show that $\Vdel$ is a sub-gamma RV and derive a bound on $\EE\Vdel$.
This can be done by bounding the higher order moments of $\Vdel$ as stated in the following lemma.

\begin{restatable}{lemma}{lemtermIqnormbound}
\label{lemma:termI_bound_qnorm_vdel}
Let $Z$, $Z_{-i}$ be defined as in \eqref{eq:Z_and_Zi_for_Rdel}, and let $\Vdel = \sum_{i=1}^n \bra{ Z - Z_{-i} }^2$.
Then for any real $q \ge 1/2$ and integer $n \ge 2$, the following holds
\begin{align}
\norm{\Vdel}_{2q} & 
~ \le ~ 
 \sfrac{2}{n^2} \sum_{i=1}^n \norm{\loss{ \Alg(\sS_n^{-i} ) , X_i}}_{4q}^2 + 2 n \beta_{4q}^2(n-1)\,,
\label{eq:termI_bound_qnorm_vdel}
\end{align}
and, in particular, $\EE\Vdel ~ \le ~  \sfrac{2}{n^2}\sum_{i=1}^n\norm{\loss{ \Alg(\sS_n^{-i} ) , X_i}}_{2}^2 + 2 n\beta_{2}^2(n-1)$.
\end{restatable}

The proof is given in \cref{apx:lemma:termI_bound_qnorm_vdel}.
\cref{lemma:termI_bound_qnorm_vdel} gives the desired upper bound for the higher order moments of $\Vdel$ including
the upper bound for $\EE\Vdel$.
To use \cref{lemma:exp_efron_stein}, it remains to show that $\Vdel$ is a sub-gamma RV according to the characterization
in \cref{theorem:moments_charac_subgamma}.
As happens, \cref{assump:one}, stated earlier, is sufficient to achieve this.

\begin{corollary}
\label{cor:term_1}
Using the previous definitions, and under \cref{assump:one}, $\Vdel \in \Gamma(v_1,c_1)$, where 
$v_1 = 4(1.1u_1 + 0.53w_1^2)$ and $c_1 = 1.46w_1$.
\end{corollary}

\noindent
The statement of \cref{cor:term_1} follows from \cref{lemma:termI_bound_qnorm_vdel}, and using 
\cref{assump:one} and \cref{theorem:moments_charac_subgamma}.
Plugging the result of \cref{cor:term_1} into  
\cref{lemma:exp_efron_stein} (which is possible because $\ell$ takes values in $[0,1]$) gives the desired final upper bound for Term I in the RHS of 
\eqref{ineq:rdel_deviation_3mainterms}. 

\begin{lemma}
\label{lemma:termI_rdel_exp_tail_bound}
Suppose that \cref{assump:one} holds and $n \ge 2$.  
Then for any $\delta \in (0,1)$ and $a>0$, with probability $1 - \delta$ the following holds
\begin{align*}
|\EE\Rdel(\Alg,\sS_n) - \Rdel(\Alg,\sS_n)| 
\le \sfrac23(1.46aw_1 + \sfrac1a)\logtwodelta + 2\sqrt{\bra{n\beta_{2}^2(n-1) + \rho_1(u_1,w_1)} \logtwodelta} ~,
\end{align*}
where $\rho_1(u_1,w_1) = 2.2a^2u_1 + 1.07a^2w_1^2$.
\end{lemma}

Consider now the choice of $a$ in the context of how it may scale with $n$ and its impact on the behavior of this bound.
First, note that $u_1$ and $w_1$ are controlled by $n\beta_{4q}^2(n-1)$, and from \cref{assump:beta_vs_n}, we assume 
that $\beta_{4q}^2(n-1)$ is a nonincreasing function of $n$.
If, for example, $n\beta_2^2(n-1) \sim \sfrac{1}{n^p}$ for some $p > 0$, then 
$u_1 \sim n^{-2p}$, 
$w_1 \sim n^{-p}$, and 
$w_1 \approx \sqrt{u_1}$.
The terms in the bound that depend on $a$ scale as $\sfrac{a}{n^p} + \sfrac{1}{a}$ with $n$.
Hence, choosing $a = n^{p/2}$, or $a = w_1^{-1/2}$, makes both, the $a$ dependent term, as well as the whole bound, 
scale with $n^{-p/2}$ as a function of $n$; 
i.e. the bound scales as $w_1^{1/2}$, and $w_1^{1/2} = o(1)$ as $n \goto \infty$.
This translates to $n\beta_{4q}^2(n-1) = o(1)$ as $n \goto \infty$;  
(and in particular, $\beta_2(n-1) = o(n^{-1/2})$) which is sufficient for the consistency of $\Rdel(\Alg,\sS_n)$.
A similar condition for consistency was also identified by \citet{bousquet_elisseeff_jmlr_2002} and \citet{celisse_guedj_2016}.

\subsection{Upper Bounding Term II}
\label{subsec:bounding_term_II}

Consider now term $\text{II}$ in inequality \eqref{ineq:rdel_deviation_3mainterms}. 
This is the deviation $|\EE R\bra{\Alg(\sS_n),\fdistri{P}} - R\bra{\Alg(\sS_n),\fdistri{P}}|$.
Note that $R\bra{\Alg(\sS_n),\fdistri{P}}$ is a function of $n$ independent RVs, and therefore,  
\cref{lemma:exp_efron_stein} will be our tool to bound this deviation.
Following the steps for upper bounding Term I in the previous section, we need to define the RVs $Z$ and $Z_{-i}$, and 
show that $\Vdel$ is a sub-gamma RV.
Let the RVs $Z$ and $Z_{-i}$ be defined as follows
\begin{equation}
Z      ~ = ~ R\bra{\Alg(\sS_n),\fdistri{P}}, \qquad
Z_{-i} ~ = ~ R\bra{\Alg(\sS_n^{-i}),\fdistri{P}}.
\label{eq:Z_and_Zi_for_risk}
\end{equation}
Similar to \cref{lemma:termI_bound_qnorm_vdel} we have the following result:
\begin{restatable}{lemma}{lemtermIIqnormbound}
\label{lemma:temrII_vdel_is_subgamma}
\quad
Let $Z$ and $Z_{-i}$ be defined as in \eqref{eq:Z_and_Zi_for_risk} and let
$\Vdel = \sum_{i=1}^n(Z - Z_{-i})^2$.
Then for any real $q \ge 1/2$, and $n \ge 2$, the following holds
\begin{align}
\norm{\Vdel}_{2q}
& ~ \le ~
n \beta^2_{4q}(n) 
~,
\label{eq:termII_expect_vdel_upper_bound}
\end{align}
and, in particular, $\EE\Vdel ~ \le ~ n\beta_2^2(n)$.
\end{restatable}
\noindent
By \cref{assump:beta_vs_n}, $n \mapsto \beta_q(n)$ is nonincreasing. This, combined with \cref{assump:one} gives the 
following result, which parallels \cref{cor:term_1}:
\begin{corollary}
\label{cor:term_2}
Using the previous definitions, and under \cref{assump:beta_vs_n,assump:one}, $\Vdel \in \Gamma(v_1,c_1)$, where 
$v_1 = 4(1.1u_1 + 0.53w_1^2)$ and $c_1 = 1.46w_1$.
\end{corollary}
The steps to derive the final bound for Term $\text{II}$ are exactly the same derivation steps for the previous bound.
The final bound is given by the following lemma which simply plugs in the results of 
\cref{lemma:temrII_vdel_is_subgamma} and \cref{cor:term_2} into \cref{lemma:exp_efron_stein} (the conditions of the latter lemma are met thanks to the assumption that $\ell$ takes values in $[0,1]$).
\begin{lemma}
\label{lemma:termII_risk_exp_upper_bound}
Suppose that \cref{assump:beta_vs_n,assump:one} hold and  $n \ge 2$.  
Then, for any $\delta \in (0,1)$ and $a > 0$, with probability $1 - \delta$ the following holds
\begin{align*}
&\abs{ \EE R\bra{\Alg(\sS_n),\fdistri{P}} - R\bra{\Alg(\sS_n),\fdistri{P}} } 
\le 
\sfrac{2}{3}(1.46a w_1 + \sfrac{1}{a})\logtwodelta 
+
2 \sqrt{( n\beta_2^2(n) + \rho_1(u_1,w_1))\logtwodelta}
~,&
\end{align*}
where, as before, $\rho_1(u_1,w_1) = 2.2a^2u_1 + 1.07a^2w_1^2$.
\end{lemma}
Concerning the choice of $a$, the discussion after \cref{lemma:termI_rdel_exp_tail_bound} applies.

\subsection{Upper Bounding Term III}
\label{subsec:upper_bound_termIII}

For term $\text{III}$ in inequality \eqref{ineq:rdel_deviation_3mainterms} there are no random quantities to account for 
since both terms in the modulus are expectations of RVs.
Hence, an upper bound on this deviation will always hold. 
\begin{restatable}{lemma}{lemtermIIIupperbound}
\label{lemma:termIII_upper_bound}
Using the previous setup and definitions, let $\Alg$ be a learning rule with $L_2$ stability coefficient $\beta_2(n)$.
Then for $n \ge 2$, the following holds
\begin{align}
|\EE R(\Alg(\sS_n),\fdistri{P}) - \EE\Rdel(\Alg,\sS_n)|
~ \le ~
\beta_1(n) \le \beta_2(n)~.
\end{align}
\end{restatable}

\subsubsection{Proof of \cref{theorem:exp_tail_bound_for_del}}
\label{subsubsec:proof_theorem_exp_tail_bound_for_del}

At this point, we have obtained the three desired upper bounds for each term in the RHS of inequality 
\eqref{ineq:rdel_deviation_3mainterms}.
The proof of \cref{theorem:exp_tail_bound_for_del} starts by plugging the results of 
\cref{lemma:termI_rdel_exp_tail_bound}, 
\cref{lemma:termII_risk_exp_upper_bound}, and
\cref{lemma:termIII_upper_bound}
into inequality \eqref{ineq:rdel_deviation_3mainterms} and then simplifying the expression to improve the presentation 
of the final result.

For unbounded losses, one can repeat the steps of this proof, with the exception that instead of the removal version $\Vdel$ of the variance proxy, one should use the ``classic'' variance proxy, $V$ (from \cref{theorem:efron_stein_main}) everywhere $\Vdel$ is used. Note that for $V$, the exact analogue of \cref{theorem:cgf_z_vdel_expbound} can be shown to hold as shown in \cref{theorem:cgf_z_v_expbound} of \cref{sec:theorem:cgf_z_v_expbound}. 
Now, one can show that 
\cref{lemma:termI_bound_qnorm_vdel} continues to hold with the removal variance proxy replaced with the classic variance proxy if the RHS in the display of this lemma is multiplied by $4$
(see \cref{lemma:termI_bound_qnorm_v} in \cref{sec:lemma:termI_bound_qnorm_v}).
The same holds for term II (see \cref{lemma:temrII_vdel_is_subgamma} in \cref{sec:lemma:temrII_v_is_subgamma}).
These changes mean that if in the display of \cref{assump:one}, the LHS is multiplied by four, the rest of the proof goes through without any further changes.

\section{Example (Application to Unbounded Ridge Regression)}
\label{subsec:example_ridgeregression}

In this section we apply the exponential tail bound in \cref{theorem:exp_tail_bound_for_del} to the ridge 
regression rule with bounded covariates and \emph{unbounded response variables}. 
Note that in the presence of unbounded response variables, ridge regression is \emph{not} uniformly stable.
In particular, the bound of \citet{bousquet_elisseeff_jmlr_2002} is not directly applicable in this setting.%
\footnote{As noted earlier, one approach to save this is to use a case-based analysis, where one case is that  some of the response variables are above a threshold to be chosen later, the other case is that they are all below a threshold. The probability of the first case can be kept below $\delta$, by choosing the threshold high enough. The price of this compared to the bound below is increased constants, and also an extra $\log(n)$ factor.}
We follow the setup of \citet{celisse_guedj_2016} (except that we allow unbounded response variables) and we will borrow some results from their work.
Let the data be $(\vx_1,Y_1), \dots, (\vx_n,Y_n)$, $\vx_i\in \RR^d$, $Y_i\in \RR$ ($1\le i\le n$), and fix $\lambda>0$.
The ridge regression estimator $\Alg_{\lambda}$ is defined via 
\begin{align}
\Alg_{\lambda}(\sS_n) 
&= 
\argmin_{\vw \in \RR^d}
\curbra{\frac{1}{n} \sum_{i=1}^n (Y_i - \vw^\top\vx_i)^2 + \lambda\norm{\vw}^2_2 } \nonumber \\
& = 
(\est{\bogs{\Sigma}} + n \lambda\mI_d)^{-1} \mX^\top\vy ~, \label{eq:ridge_estimator}
\end{align}
where $\mX$ is the $n\times d$ matrix obtained by stacking the $d$-dimensional vectors 
$\vx_1^\top,\dots, \vx_n^\top$,
$\est{\bogs{\Sigma}} = \mX^\top\mX$ is the (unnormalized) sample covariance matrix, and $\vy = [Y_1,\dots,Y_n]^\top$ 
is the vector of response variables.
The loss $\loss{\cdot}$ is the quadratic loss: $\loss{ \vw, (\vx,y) } = (\vw^\top x - y)^2$.
As usual, we assume that the data is i.i.d. from some common distribution.
For the purpose of this example, 
we have the following two assumptions on this distribution:

\begin{assumption}
\label{assump:xbd}
$\exists$ $0 < B_X < +\infty$ s.t.  $\norm{\vx_1} \leq B_X$ a.s. 
\end{assumption}

\begin{assumption}
\label{assump:ybd}
$\exists$ $u_Y,w_Y\ge 0$ s.t. $\forall q\ge 1$, $\norm{Y_1^4}_{2q} \le \sqrt{q u_Y} \vee q w_Y$.
\end{assumption}
Note that this last assumption allows unbounded responses, as long as their $4$th moment is subgamma. For example, $Y_1 = \sqrt{|Z|} \sgn(Z)$ with a gaussian $Z$ satisfies this condition.

To use \cref{theorem:exp_tail_bound_for_del}, the $L_q$ stability coefficient for the ridge estimator, or an upper bound 
on it, needs to be calculated.
This is given in the next theorem taken from the paper of \citet{celisse_guedj_2016}.
Their result is applicable because
the ridge regression estimator is symmetric, and hence, our definition for the $L_q$ stability coefficients then coincides with theirs, as it was noted earlier.%
\footnote{The result is streamlined by choosing the value of $\eta$ in their result to minimize the upper bound on the 
stability coefficient.}
\begin{theorem}
\label{theorem:lq_stability_ridgeregression}
Let $\Alg_{\lambda}$ be the ridge estimator in \cref{eq:ridge_estimator}
and let \cref{assump:xbd} hold.
Then, for any sample size $n > 1$, as long as $s_{\lambda,n} = \lambda-\sfrac{1}{n-1}>0$,
for any $q \geq 1$, 
$\Alg_{\lambda}$ is $L_q$ stable with the following bound on its stability:
\begin{align}
\beta_q(\Alg_{\lambda},\ell,\fdistri{P},n)
& \le
2 \norm{Y_1}_{2q}^2 
	\frac{B^2_X}{n\lambda} 
	\bra{1 + \frac{B_X^2 + \lambda}{s_{\lambda,n}}}
	\bra{1 + \frac{B_X^2}{\lambda}} ~.
\label{eq:cng_lq_upperbound_ridgereg}
\end{align}
\end{theorem}
To simplify the expression for the upper bound in \cref{eq:cng_lq_upperbound_ridgereg}, let 
\begin{align}
\kappa 
& = 
2 \frac{B^2_X}{\lambda} \bra{1 + \frac{B_X^2 + \lambda}{s_{\lambda,n-1}}} \bra{1 + \frac{B_X^2}{\lambda}} .
\label{eq:bound_constterm}
\end{align}
Then,
$\beta_2(n) \le \sfrac{\kappa}{n} \norm{Y_1^2}_{2}$, $n \beta_2^2(n-1) \le n \sfrac{\kappa^2}{(n-1)^2} \norm{Y_1^2}_2^2$,
Furthermore,  
$
\beta_q(n-1) \le  \sfrac{\kappa}{n-1} \norm{Y_1}_{2q}^2
$ and, hence
\begin{align*}
n \beta_{4q}^2(n-1) 
\le \sfrac{\kappa^2}{n-1} \norm{Y_1}_{8q}^4 
= \sfrac{\kappa^2}{n-1} \norm{Y_1^4}_{2q}\,.
\end{align*}
Some calculations gives (cf. \cref{sec:ridgecalc})
\begin{align}
\sfrac{2}{n} \norm{\loss{ \Alg_\lambda(\sS_n^{-1}), (x_1,Y_1) }  }_{4q}^2
\le 
\sfrac{4\norm{Y_1^4}_{2q}}{n} \bra{1+\frac{B_X^4}{\lambda^2}}^2\,.
\label{eq:lossmomentbound}
\end{align}
Thus,
\begin{align*}
8n \beta_{4q}^2(n-1)  + 
\sfrac{8}{n} \norm{\loss{ \Alg_\lambda(\sS_n^{-1}), (x_1,Y_1) }  }_{4q}^2
\le 
\sfrac{\norm{Y_1^4}_{2q}}{n-1}  \hat{\kappa}
\end{align*}
where
\begin{align*}
\hat{\kappa} = 8\bra{ \kappa^2 + 2\bra{1+\frac{B_X^4}{\lambda^2}}^2}~.
\end{align*}
Thus, to meet the modified \cref{assump:one} where the LHS of \cref{eq:assumpone} is multiplied by $4$, we can choose $u_1 = \frac{\hat{\kappa}^2}{(n-1)^2} u_Y $ and $w_1 = \frac{\hat{\kappa}}{(n-1)}w_Y$.
Note that $\hat{\kappa}$ only depends on $B_X$ and $\lambda$, but is independent of $n$. 
In particular $\hat{\kappa}$ scales with $1/\lambda^6$ ($\kappa$ scales with $1/\lambda^3$).
We can now plug into 
the simplified version \eqref{eq:simplifiedbound} of 
the bound of \cref{theorem:exp_tail_bound_for_del} 
to obtain an exponential tail bound for the deleted estimate for ridge regression.

\begin{restatable}{corollary}{corexptailboundforridge}
\label{cor:exp_tail_bound_for_ridge}
Given all definitions above, let $\Rdel\bra{\Alg_{\lambda},\sS_n}$ be the deleted estimate for the ridge regression rule,
$R(\Alg_{\lambda}(\sS_n),\fdistri{P})$ be its risk, and assume that \cref{assump:xbd} and \cref{assump:ybd} hold.
Further, let $\mu = \norm{Y_1^2}_2$.
Then, for $\delta \in (0,1)$,
with probability $1 - \delta$ the following holds
\begin{align*}
\MoveEqLeft
| R(\Alg_{\lambda}(\sS_n),\fdistri{P}) - \Rdel\bra{\Alg_{\lambda},\sS_n} |
 \le
\frac{\kappa\mu}{n} + 4 \kappa\mu  \sqrt{\sfrac{n}{(n-1)^2}\logtwodelta} +\\
&  \qquad
+ 8 
\sqrt{ 
 \sfrac{\hat{\kappa}}{3(n-1)} 
 \bra{   \sqrt{ (2.2 u_Y + 1.07 w_Y^2)} + 
\sfrac{1}{3} 1.46 w_Y }
 } \logtwodelta \,.
 \numberthis
\label{eq:exp_tail_bound_for_ridge}
\end{align*}
\end{restatable}

Note that as far as we know this is the first bound for the deleted estimate for ridge regression which allows unbounded response variables.%
The proof of \cref{cor:exp_tail_bound_for_ridge} is straightforward and is hence omitted. 
As we see the bound scales with $1/\sqrt{n}$ regardless the value of $\lambda$.
However, the bound scales quite poorly with $1/\lambda$. This poor scaling is not inherent to ridge regression but follows from the (oversimplified) analysis. However, for now, we leave it to future work to address this defect of our bound. Finally, let us note that
while not shown here, a similar bound is available for the resubstitution estimate: The $\gamma_q$ coefficients show a behavior similar to the $\beta_q$ coefficients.

The reader may also be wondering about how the presented bound compares with that presented by
\cite{celisse_guedj_2016} in their Theorem~4.
Unfortunately, this comparison is meaningless as the bounds here are incorrect.
The problem originates in Proposition~3 where on the right-hand side some terms (corresponding to \eqref{eq:lossmomentbound}) are missing: In the proof, the authors incorrectly use 
$(w^\top x - y)^2 - (w^\top x' - y')^2 = (w^\top (x-x')+y'-y) (w^\top (x +x') - y-y')$:
It appears that in their calculations, \citeauthor{celisse_guedj_2016} 
have accidentally dropped the $y'-y$ term from the first term on the RHS.
After correcting for this,
the $L_q$ norm of $Y$ will appear on the right-hand side in the inequality stated in this proposition,
corresponding to the bound \eqref{eq:lossmomentbound}.
We believe that after the mistakes are corrected, one will arrive at a bound that will near identical to ours.

\section{Concluding Remarks}
\label{sec:conclusion}

In this work we consider the gap between two regimes of stability-based generalization results;
(\emph{i}) exponential generalization bounds based on strong notions of stability which are distribution independent and 
computationally intractable, such as uniform stability, and
(\emph{ii}) polynomial generalization bounds based on weaker notions of stability but are distribution dependent and 
computationally tractable such as hypothesis stability and $L_q$ stability.
Using the exponential Efron-Stein inequality we were able to bridge this gap by deriving an exponential concentration 
bound for $L_q$ stable learning rules, where the loss of the learning rules is expressed in terms of the deleted estimate.

We believe that our result is one step forward on two fronts;
(\emph{i}) computing empirical tight confidence intervals for the expected loss of a learning rule where the confidence 
interval holds with high probability; and
(\emph{ii}) understanding the role of stability in the concentration of different empirical loss estimates around their 
expectations (in supervised and unsupervised learning).
For instance, it will be interesting to understand how the stability of a learning rule can guide our choice for $k$, 
and hence the fold size, for the KFCV estimate, such that the estimate concentrates well around the expected risk.
Last, we second on the question posed by \citet{bousquet_elisseeff_jmlr_2002}, of whether it is possible to design 
algorithms that can maximize their own stability while gaining also on performance.

\subsection*{Acknowledgments}

We would like to thank our ALT Reviewers and AC for their thoughtful comments which helped us improve the presentation 
of our manuscript.

\bibliography{karim}

\appendix

\newpage

\section{Proof of \cref{cor:efron_stein_remove}}
\label{appx:cor_efron_stein_remove}

\corefronsteinremove*

\begin{proof}
For any RV $X$, we have the following fact: $\VV[X] \leq \EE[(X - a)^2]$, for any $a \in \RR$.
Assume that $ \expectop_{-i}[Z_{-i}]$ exists for all $1\leq i \leq n$, and applying the previous fact conditionally for 
 $\fset{S}_n^{-i}$, then $\EE_{-i}[(Z - \EE_{-i}Z)^2] \leq \EE_{-i}[(Z - Z_{-i})^2]$.
Taking expectations and summing over all $i$ we get that $\EE V \leq \EE \Vdel$.
Combining the Efron-Stein inequality for RV $Z$ with the previous inequality, we get the desired result. 
\end{proof}

\section{Proof of \cref{theorem:cgf_z_vdel_expbound}}
\label{appx:theorem_cgf_z_vdel_expbound}

\thmcgfzvdelexpbound*

\begin{proof}
The proof of this theorem relies on the result of Theorem 6.6 in \citep{book_concentration_2013} which we state here for 
convenience as a proposition without proof.
\begin{proposition}
Let $\phi(u) = e^u - u -1$.
Then for all $\lambda \in \mathbb{R}$,
\begin{align}
\label{ineq:mod_log_sobolev_del}
\lambda \expectone{Z \exp(\lambda Z)}
-
\expectone{\exp(\lambda Z)}
\log
\expectone{\exp(\lambda Z)}
& \le
\sum_{i=1}^n
\expectone{
\exp(\lambda Z)
\phi\bra{- \lambda (Z - Z_{-i})}
} ~.
\end{align}
\end{proposition}
To make use of inequality \eqref{ineq:mod_log_sobolev_del}, we need to establish an appropriate upper bound for the RHS
of \eqref{ineq:mod_log_sobolev_del}.
Note that for $u \leq 1$, $\phi(u) \leq u^2$.
By assumption $|Z - Z_{-i}| \leq 1$ holds almost surely. 
Since $0 < \lambda \leq 1$, we get
\begin{align*}
\sum_{i=1}^n
\expectone{ \exp(\lambda Z) \phi\bra{- \lambda (Z - Z_{-i})} }
&\leq
\lambda^2
\sum_{i=1}^n
\expectone{ \exp(\lambda Z) \bra{Z - Z_{-i}}^2 }
\\
&=
\lambda^2 \expectone{\Vdel \exp(\lambda Z) } ~.
\end{align*}
It follows that \eqref{ineq:mod_log_sobolev_del} can be written as
\begin{align}
\label{ineq:mod_log_sobolev_del_f1}
\lambda \expectone{Z \exp(\lambda Z)}
-
\expectone{\exp(\lambda Z)}
\log
\expectone{\exp(\lambda Z)}
& \le
\lambda^2 \expectone{ \exp(\lambda Z) \Vdel} ~.
\end{align}
The RHS of the previous inequality has two coupled random variables; $\exp(\lambda Z)$ and $\Vdel$.
To make use of \eqref{ineq:mod_log_sobolev_del}, we decouple the two random variables using the following useful tool
from \citep{massart_2000} which we state as a proposition without a proof.

\begin{proposition}
For random variable $W$, and for any $\lambda \in \mathbb{R}$, if $\expectone{\exp(\lambda W)} < \infty$, then the
following holds
\begin{align}
\label{ineq:decouple_tool}
\frac{\EE{\lambda W \exp(\lambda Z)} }{\EE{\exp(\lambda Z)}}
& \le
\frac{\EE{\lambda Z \exp(\lambda Z)}}{\EE{\exp(\lambda Z)}}
-
\log \EE\,{\exp(\lambda Z)}
+
\log \EE\,{\exp(\lambda W)} ~.
\end{align}
\end{proposition}
Multiplying both sides of \eqref{ineq:decouple_tool} by $\EE{\exp(\lambda Z)}$ and replacing $W$ with $\Vdel / \theta$
we get that:
\begin{align}
\label{ineq:decople_tool_f1}
\EE{\exp(\lambda Z) \Vdel}
& \le
\theta
\sqbra{
\EE{Z \exp(\lambda Z)}
-
\frac{1}{\lambda}\EE{\exp(\lambda Z)}\log \EE{\exp(\lambda Z)}
+
\frac{1}{\lambda}\EE{\exp(\lambda Z)}\log \EE{\exp\bra{ \lambda \frac{\Vdel}{\theta}} }
}.
\end{align}
Introduce $F(\lambda) = \EE{\exp(\lambda Z)}$, and $G(\lambda) = \log\EE{\exp(\lambda \Vdel)}$.
Note that $F'(\lambda) = \EE{Z \exp(\lambda Z)}$.
Plugging \eqref{ineq:decople_tool_f1} into \eqref{ineq:mod_log_sobolev_del_f1} and using the compact notation
$F(\lambda)$, $F'(\lambda)$, and $G(\lambda / \theta)$ we get that:
\begin{align}
\lambda F'(\lambda) - F(\lambda)\log F(\lambda)
& \le
\lambda^2 \theta
\bra{
 F'(\lambda)
 -
 \frac{1}{\lambda}
 F(\lambda)
 \log F(\lambda)
 +
 \frac{1}{\lambda}
  F(\lambda)
  G(\lambda / \theta)
} ~.
\end{align}
Dividing both sides by $\lambda^2 F(\lambda)$ and rearranging the terms:
\begin{align}
\frac{1}{\lambda}
\frac{F'(\lambda)}{F(\lambda)}
-
\frac{1}{\lambda^2}
\log F(\lambda)
& \le
\frac{\theta G(\lambda / \theta)}{\lambda (1 - \lambda \theta)} ~.
\label{eq:diffeqpre}
\end{align}
The rest of the proof continues exactly as the proof of Theorem 2 from \citep{boucheron_lugosi_massart_2003}:
As the left-hand side of the above display is just the derivative of 
$H(\lambda) = \frac1{\lambda} \log F(\lambda)$, \eqref{eq:diffeqpre} is equivalent to 
$H'(\lambda) \le \frac{\theta G(\lambda / \theta)}{\lambda (1 - \lambda \theta)}$.
Recalling that $\lim_{\lambda\to 0+} H(\lambda) = \EE[Z]$, the integration of the differential inequality gives
$H(\lambda) \le \EE[Z] + \theta\, \int_0^\lambda \frac{ G(s / \theta)}{s (1 - s \theta)} \,ds$.
Notice that $G$ is convex. 
This implies that the integrand is a nondecreasing function of $s$ and therefore
$\log F(\lambda) \le \lambda \EE[Z] + \frac{\lambda \theta G(\lambda/\theta)}{1-\lambda \theta}$.
\end{proof}

\section{Proof of \cref{lemma:exp_efron_stein}}
\label{appx:lemma_exp_efron_stein}

\lemmaexpefronstein*

\begin{proof}
Since $\Vdel-\EE\Vdel \in \Gamma_+(v,c)$, for any $\lambda \in (0,1/c)$ we have
\begin{align*}
\psi_{ \Vdel - \EE{\Vdel} }(\lambda)
=
\log \expectone{ \exp( \lambda (\Vdel - \expectop\Vdel) ) }
\le
\frac{\lambda^2 v}{2(1 - c\lambda)} ~.
\end{align*}
Rearranging the terms we get
\begin{align}
\log
\expectone{\exp(\lambda \Vdel)}
\le
\lambda
\EE{\Vdel}
+
\frac{\lambda^2 (v/2)}{1 - c \lambda} ~.
\label{eq:vdel_subgamma_bound}
\end{align}
Combining this with the result of \cref{theorem:cgf_z_vdel_expbound} where we choose $\theta=1$, we get
\begin{align}
\psi_{Z-\EE Z}(\lambda)
\le
\frac{\lambda}{1 - \lambda}
\bra{ \lambda \EE{\Vdel} + \frac{\lambda^2 (v/2)}{1 - c \lambda} } ~.
\label{eq:apply_theorem_cgf_Z}
\end{align}
We upper bound the term on the right-hand side as follows
\begin{align*}
\frac{\lambda}{1 - \lambda}
\bra{ \lambda \EE{\Vdel} + \frac{\lambda^2 (v/2)}{1 - c \lambda} }
& =
\frac{\lambda}{1 - \lambda}
\bra{ \frac{\lambda \EE{\Vdel} - c\lambda^2\EE{\Vdel} + \lambda^2v/2}{(1 - c\lambda)} }
\\
& \le
\frac{\lambda}{1 - \lambda}
\bra{ \frac{\lambda \EE{\Vdel} + \lambda^2 (v/2) }{(1 - c\lambda)} }
\\
& =
\frac{\lambda^2 \EE{\Vdel} + \lambda^3 (v/2) }{(1 - \lambda)(1 - c\lambda)}
\\
& \le
\frac{\lambda^2 \EE{\Vdel} + \lambda^2 (v/2) }{(1 - \lambda)(1 - c\lambda)}
\\
& =
\frac{\lambda^2 (\EE{\Vdel + v/2}) }{(1 - \lambda)(1 - c\lambda)}
\\
& \le
\frac{\lambda^2 (\EE{\Vdel + v/2}) }{(1 - (c+1)\lambda)} ~,
\end{align*}
where the last inequality holds provided that $0<\lambda<1/(c+1)$.
Thus we finally get that
\begin{align}
\psi_{Z-\EE Z}(\lambda)
\le
\frac{\lambda^2 (\EE{\Vdel + v/2}) }{(1 - (c+1)\lambda)} ~.
\label{eq:cgf_Z_final_bound}
\end{align}
Recall that the Cramer-Chernoff method gives that for any $\lambda>0$,
\begin{align*}
\probone{Z > \EE{Z} + t}\le \exp(-(\lambda t - \psi_{Z-\EE Z}(\lambda) )) ~.
\end{align*}
This combined with \eqref{eq:cgf_Z_final_bound}, we see that we need to lower bound 
\begin{align*}
\lambda t-  \psi_{Z-\EE Z}(\lambda)
\ge \lambda t - \frac{\lambda^2 (\EE{\Vdel + v/2}) }{(1 - (c+1)\lambda)} ~,
\end{align*} 
where $\lambda\in (0,1]\cap (0,1/(c+1)) = (0,1/(c+1))$ can be chosen so that the lower bound is the largest.
From Lemma~11 of \citet{boucheron_lugosi_massart_2003}, we have that for any $p,q > 0$,
\begin{align*}
\sup_{\lambda \in [0, 1/q)}
\bra{ \lambda t - \frac{\lambda^2 p}{1 - q \lambda}}
\ge
\frac{t^2}{4p + 2q (t/3)} ~,
\end{align*}
and the supremum is attained at
\begin{align*}
\lambda
=
\frac{1}{q}
\bra{1 - \bra{1 + \frac{qt}{p}}^{-1/2}} ~.
\end{align*}
Setting $p = \EE{\Vdel + v/2}$, $q = c+1$, we see that the optimizing $\lambda$ belongs to $(0,1/(c+1))$. 
Hence,
\begin{align*}
\probone{Z > \EE{Z} + t}
~ \le ~
\exp\bra{\frac{-t^2}{4 (\EE{\Vdel} + v/2) + 2(c+1)t/3 }} ~.
\end{align*}
The previous inequality gives an exponential bound on the upper tail for the deviation of the RV $Z$ from its expectation.

Finally, letting the right hand side of the previous inequality to equal $\delta$ and solving for $t$ then after some 
further upper bounding to simplify the resulting expression 
(in particular, using $\sqrt{|a|+|b|} \le \sqrt{|a|} + \sqrt{|b|}$) and using a union bound to obtain a two-sided tail inequality, we get
\begin{align}
\abs{Z - \EE Z}
~ \le ~
\sfrac{2}{3}(c+1)\logtwodelta
+
2\sqrt{(\EE\Vdel + v/2) \logtwodelta} ~.
\label{eq:zezbound}
\end{align}
The result now follows by applying \eqref{eq:zezbound} to $Z' = aZ$, $Z_{-i}' = aZ_{-i}$ and 
$\Vdel' = \sum_i (Z'- Z_{-i}')^2$.
Noting that $\Vdel' = a^2 \Vdel \in \Gamma(a^4 v, a^2c)$, we get
\begin{align*}
a \abs{Z - \EE Z}
~ \le ~
\sfrac{2}{3}(a^2 c+1)\logtwodelta
+
2\sqrt{(a^2 \EE\Vdel + a^4 v/2) \logtwodelta} ~.
\end{align*}
Dividing both sides by $a$ gives the desired inequality.
\end{proof}

\section{Proof of \cref{lemma:termI_bound_qnorm_vdel}}
\label{apx:lemma:termI_bound_qnorm_vdel}

\lemtermIqnormbound*

\begin{proof}
Let $q \ge 1$. Then,
\begin{align}
\norm{\Vdel}_q
& =
\norm{ \sum_{i=1}^n (Z - Z_{-i})^2 }_q
 \le
\sum_{i=1}^n \norm{ (Z - Z_{-i})^2 }_q
\,,
\label{eq:vdelbasic}
\end{align}
where the inequality is by the triangle inequality.
Now, using the definitions of $Z$ and $Z_{-1}$ (cf. \cref{eq:Z_and_Zi_for_Rdel}),
\begin{align*}
\MoveEqLeft (Z - Z_{-1})^2 \\
& = 
\bra{\frac{1}{n}  \loss{ \Alg(\sS_n^{-1} ) , X_1}
+
\frac{1}{n}\sum_{i=2}^n \bra{\loss{ \Alg(\sS_n^{-i} ) , X_i}  -  \loss{\Alg(\sS_n^{-1 , -i}), X_i} }}^2 
\\
& 
\le
\frac{2}{n^2}  \losssq{ \Alg(\sS_n^{-1} ) , X_1}
+
2 \bra{\frac{1}{n}\sum_{i=2}^n \bra{\loss{ \Alg(\sS_n^{-i} ) , X_i}  -  \loss{\Alg(\sS_n^{-1 , -i}), X_i} }}^2 
\quad\mathtxtsm{($(a+b)^2 \le 2a^2 + 2b^2$)}
\\
& 
\le
\frac{2}{n^2}  \losssq{ \Alg(\sS_n^{-1} ) , X_1}
+
2 \frac{1}{n}\sum_{i=2}^n \bra{\loss{ \Alg(\sS_n^{-i} ) , X_i}  -  \loss{\Alg(\sS_n^{-1 , -i}), X_i} }^2 \,.
\quad\mathtxtsm{(Jensen's inequality)}
\end{align*}
Taking the $q$-norm of both sides, 
using the triangle inequality and that for any $U$ RV, $\norm{U^2}_q = \norm{U}_{2q}^2$, 
we get
\begin{align}
\norm{ (Z - Z_{-1})^2 }_q 
& \le
\frac{2}{n^2}  \norm{\loss{ \Alg(\sS_n^{-1} ) , X_1}}_{2q}^2
+
\frac{2}{n}\sum_{i=2}^n \norm{\loss{ \Alg(\sS_n^{-i} ) , X_i}  -  \loss{\Alg(\sS_n^{-1 , -i}), X_i} }_{2q}^2\,.
\label{eq:RdelZminusZminusone}
\end{align}
An analogous inequality holds for $\norm{(Z-Z_{-j})^2}_q$ with $j>1$. 
Summing up all these, 
using that $(\Alg(\sS_n^{-j} ) , X_j)_{j}$ share the same distribution
and combining
with \cref{eq:vdelbasic}, we get
\begin{align*}
\norm{\Vdel}_q
 & \le \frac{2}{n^2} \sum_{j=1}^n \norm{\loss{ \Alg(\sS_n^{-j} ) , X_j}}_{2q}^2 +
  2 \frac{1}{n} \sum_{j=1}^n
\sum_{i\ne j} \norm{\loss{ \Alg(\sS_n^{-i} ) , X_i}  -  \loss{\Alg(\sS_n^{-j , -i}), X_i} }_{2q}^2\\
 & = \frac{2}{n^2} \sum_{j=1}^n  \norm{\loss{ \Alg(\sS_n^{-j} ) , X_j}}_{2q}^2 +
  2 \sum_{i=1}^n
\underbrace{\frac{1}{n} \sum_{j\ne i}  
 \norm{\loss{ \Alg(\sS_n^{-i} ) , X_i}  -  \loss{\Alg(\sS_n^{-i , -j}), X_i} }_{2q}^2}_{\beta_{2q}^2(n-1)}\\
& = \frac{2}{n^2} \sum_{i=1}^n \norm{\loss{ \Alg(\sS_n^{-i} ) , X_i}}_{2q}^2 + 2 n \beta_{2q}^2(n-1)\,.
\end{align*}
Replacing $q$ with $2q$ gives the desired result.
\end{proof}

\section{Proof of \cref{lemma:temrII_vdel_is_subgamma}}
\label{apx:lemma:temrII_vdel_is_subgamma}

\lemtermIIqnormbound*

\begin{proof}
Let $q\ge 1$. Then, similar to the previous proof,
\begin{align}
\norm{\Vdel}_q
=
\norm{ \sum_{i=1}^n \bra{Z - Z_{-i}}^2 }_q
\le
\sum_{i=1}^n
\norm{ \bra{Z - Z_{-i}}^2 }_q
= 
\sum_{i=1}^n
\norm{ \bra{Z - Z_{-i}} }_{2q}^2\,,
\label{eq:vdel2exp}
\end{align}
where the last inequality is because for any RV $U$, $\norm{U^2}_q = \norm{U}_{2q}^2$.
Then, using the definitions of $Z$ and $Z_{-1}$,
\begin{align}
\norm{Z-Z_{-1}}_{2q}^2
& \le
\norm{
R(\Alg(\sS_n),\fdistri{P}) - R(\Alg(\sS_n^{-1}),\fdistri{P})
}_{2q}^2
\nonumber
\\
& =
\norm{
\expectone{ \loss{\Alg(\fset{S}_n),X} - \loss{\Alg(\fset{S}_n^{-1}),X} | \fset{S}_n }
}_{2q}^2
\tag{\mathtxtsm{tower rule}}
\nonumber
\\
& =
\expectone{ 
|\expectone{ \loss{\Alg(\fset{S}_n),X} - \loss{\Alg(\fset{S}_n^{-1}),X} | \fset{S}_n }|^{2q}
}^{2/(2q)}
\nonumber
\\
& \le
\expectone{ 
\expectone{ |\loss{\Alg(\fset{S}_n),X} - \loss{\Alg(\fset{S}_n^{-1}),X}|^{2q}\, | \fset{S}_n }
}^{2/(2q)}
\tag{\mathtxtsm{Jensen's inequality}}
\nonumber
\\
& =
\expectone{ 
|\loss{\Alg(\fset{S}_n),X} - \loss{\Alg(\fset{S}_n^{-1}),X}|^{2q}
}^{2/(2q)}
\tag{\mathtxtsm{tower rule}}
\nonumber
\\
& =
\norm{
\loss{\Alg(\fset{S}_n),X} - \loss{\Alg(\fset{S}_n^{-1}),X}
}_{2q}^{2}\,.
\label{ineq:termII_moment_vdel_final}
\end{align}
An analogous inequality holds for $\norm{Z-Z_{-i}}_{2q}^2$ with $i>1$. Summing up all these,
combining with \eqref{eq:vdel2exp} we get
\begin{align*}
\norm{\Vdel}_q \le n\,\, \frac{1}{n} \sum_{i=1}^n \norm{
\loss{\Alg(\fset{S}_n),X} - \loss{\Alg(\fset{S}_n^{-i}),X}
}_{2q}^{2} = n \beta_{2q}^2(n)~.
\end{align*}
Replacing $q$ with $2q$ yields that
\begin{align}
\norm{\Vdel}_{2q}
& ~ \le ~
n \beta^2_{4q}(n) ~.
\nonumber
\end{align}
\end{proof}

\section{Proof of \cref{lemma:termIII_upper_bound}}
\label{apx:lemma:termIII_upper_bound}

\lemtermIIIupperbound*

\begin{proof}
To derive a bound on $|\EE R(\Alg(\sS_n),\fdistri{P}) - \EE\Rdel(\Alg,\sS_n)|$ in terms of $L_q$-stability,
we proceed as follows.
First, note that
$\EE R\bra{\Alg(\sS_n),\fdistri{P}} = \expectone{ \loss{ \falgo{A}(\fset{S}_n),X} }$.
Second, for $\EE\Rdel(\Alg,\sS_n)$, we have
\begin{align}
\E \Rdel(\Alg,\sS_n)
& =
\expectone{
\frac{1}{n}
\sum_{i=1}^n
\loss{\Alg\bra{\sS_n^{-i}},X_i}
} 
\nonumber
\\
& =
\frac{1}{n}
\sum_{i=1}^n
\expectone{
\loss{ \Alg\bra{\sS_n^{-i}},X_i }
}
\nonumber
\\
& =
\frac{1}{n}
\sum_{i=1}^n
\expectone{
\loss{ \falgo{A}\bra{\fset{S}_n^{-i}} , X }
}\,,
 \tag{\mathtxtsm{by \iid\ of the examples}}
\nonumber
\end{align}
where $X \sim \fdistri{P}$ is independent of $\sS_n$.%
\footnote{Note that, of course, as is well known,
$\expectone{\loss{ \falgo{A}\bra{\fset{S}_n^{-i}} , X }}
=
\expectone{\loss{ \falgo{A}\bra{\fset{S}_n^{-1}} , X }}$ 
also holds for any $i>1$, but we will not need this identity here.}

It follows that 
\begin{align}
\MoveEqLeft
\abs{ \EE R\bra{\Alg(\sS_n),\fdistri{P}} - \EE\Rdel(\Alg,\sS_n) }
 =
\abs{ \EE[\loss{\Alg(\sS_n),X}] -  \frac1n \sum_{i=1}^n \EE[\ell(\Alg(\sS_n^{-i}),X ) ] }
\nonumber
\\
& =
\abs{ \frac1n \sum_{i=1}^n 
\expectone{
\loss{\Alg(\sS_n),X} - \ell(\Alg(\sS_n^{-i}),X)
}
}
\nonumber
\\
& \leq
\frac1n \sum_{i=1}^n 
\expectone{
\abs{
\loss{\Alg(\sS_n),X} - \ell(\Alg(\sS_n^{-i}),X)
}
}
\nonumber
\tag{\mathtxtsm{Jensen's inequality}}
\\
& = \beta_1(n) \le \beta_2(n)\,,
\label{ineq:termIII_upper_bound}
\end{align}
where the last equality uses the definition of $\beta_1$, and the last inequality uses that $\beta_q \le \beta_{q'}$ for $q\le q'$.
\end{proof}

\section{Ridge regression: Proving \cref{eq:lossmomentbound}}
\label{sec:ridgecalc}
\newcommand{\hX}{\hat{X}}
\newcommand{\hY}{\hat{Y}}
\newcommand{\tw}{\tilde{w}}
\newcommand{\Sl}{\Sigma_\lambda }
\newcommand{\tSl}{\tilde{\Sigma}_\lambda }
\newcommand{\tX}{\tilde{X}}
\newcommand{\tY}{\tilde{Y}}
For the convenience of the reader, let us restate \cref{eq:lossmomentbound}:
\begin{align*}
\sfrac{2}{n} \norm{\loss{ \Alg_\lambda(\sS_n^{-1}), (x_1,Y_1) }  }_{4q}^2
\le 
\sfrac{4
\norm{Y_1^4}_{2q}}{n} \bra{1+\frac{B_X^4}{\lambda^2}}^2\,.
\end{align*}
Introduce the shorthand $\tw = \Alg_\lambda(\sS_n^{-1})$.
We have
\begin{align*}
\loss{ \Alg_\lambda(\sS_n^{-1}), (x_1,Y_1) }  
= (x_1^\top  \tw- Y_1)^2 \le 2 (x_1x^\top \tw)^2 + 2 Y_1^2
\end{align*}
Then, $|x_1^\top \tw| \le \norm{x_1} \norm{\tw}\le B_X \norm{\tw}$ ($\norm{\cdot}$ denotes the $2$-norm).
Introduce the abbreviation $\norm{Z}_{2,q} = \norm{ \norm{Z} }_q$.
Hence,
\begin{align*}
\norm{ |x_1^\top \tw|^2}_q
\le B_X^2 \norm{ \norm{\tw}^2 }_q  = B_X^2 \norm{ \tw }_{2,2q}^2\,.
\end{align*}
 
Let $\tX = [x_2 \dots x_{n}]^\top$ (thus, $\tX\in \RR^{(n-1)\times d}$, with $x_1$ left out),
$\tY = [Y_2,\dots,Y_n]^\top$
and $\tSl = \tX^\top \tX + n\lambda I$ so that $\tw = \tSl^{-1} \tX^\top \tY$.

We calculate $\norm{\tw} \le  \norm{\tSl^{-1}} \sum_{i=2}^n |Y_i|  \norm{x_i}
\le \frac{B_X}{n\lambda} \sum_{i=2}^n |Y_i|  $ and so
\begin{align*}
\norm{\tw }_{2,2q} = \norm{  \tSl^{-1} \tX^\top \tY }_{2,2q} 
\le \frac{B_X}{\lambda} \norm{Y_1}_{2q}\,,
\end{align*}
where the second inequality used that $Y_1$ has the same distribution as $Y_i$ with $i>1$.
Putting things together,
\begin{align*}
\norm{\loss{ \Alg_\lambda(\sS_n^{-1}), (x_1,Y_n) }  }_q
& \le 
2 \frac{B_X^4}{\lambda^2} \norm{Y_1}^2_{2q} + 2\norm{Y_1}^2_{2q}
= 2\norm{Y_1}^2_{2q} \bra{1+\frac{B_X^4}{\lambda^2}}
\numberthis
\label{eq:gammaqbound}
\end{align*}
and thus
\begin{align*}
\frac2n \norm{\loss{ \Alg_\lambda(\sS_n^{-1}), (x_1,Y_n) }  }_{4q}^2
\le
\frac4n
\norm{Y_1}^4_{8q} \bra{1+\frac{B_X^4}{\lambda^2}}^2
=
\frac4n
\norm{Y_1^4}_{2q} \bra{1+\frac{B_X^4}{\lambda^2}}^2\,,
\end{align*}
finishing the proof.

\section{Proof for \cref{theorem:exp_tail_bound_for_res}}
\label{sec:resproof}

Here, we provide a proof for the exponential tail bound for the generalization gap defined using the empirical error:
\theoremfinalgenresultres*

\begin{proof}%
We show the proof for the case when $\ell$ takes values in $[0,1]$. The unbounded case requires the same modifications 
and similar calculation to what has been shown for \cref{theorem:exp_tail_bound_for_del} and is left to the reader. 
As before, 
\begin{align*}
|\Rres\bra{\Alg,\sS_n} - R\bra{\Alg(\sS_n),\cP}|
& \le
|\Rres\bra{\Alg,\sS_n} - \E \Rres\bra{\Alg,\sS_n}| \\
& \quad +
|R\bra{\Alg(\sS_n),\cP}-\E R\bra{\Alg(\sS_n),\cP}| \\
& \quad +
|\E\Rres\bra{\Alg,\sS_n} - \E R\bra{\Alg(\sS_n),\cP}|\,.
\numberthis
\label{eq:basicdecres}
\end{align*}
To control the first term, one can use the same argument as in \cref{subsec:bounding_term_I} with the difference
that we should use
\begin{equation}
Z      ~ = ~ \Rres(\Alg,\sS_n) = \frac{1}{n}\sum_{i=1}^n \loss{ \Alg(\sS_n ) , X_i}, \qquad
Z_{-i} ~ = ~ \frac{1}{n} \sum_{\overset{j=1}{\sss j \neq i}}^n \loss{\Alg(\sS_n^{-i}), X_j},
\end{equation}
As required, $Z_{-i}$ does not depend on $X_i$.
The proof of \cref{lemma:termI_bound_qnorm_vdel} presented in \cref{apx:lemma:termI_bound_qnorm_vdel} goes through verbatim with the necessary adjustments to account for the differences in the definitions of $Z$ and $Z_{-i}$.
To show the differences encountered, note that 
\begin{align*}
Z - Z_{-1}
& = 
\frac{1}{n}  \loss{ \Alg(\sS_n ) , X_1}
+
\frac{1}{n}\sum_{i=2}^n \bra{\loss{ \Alg(\sS_n) , X_i}  -  \loss{\Alg(\sS_n^{-1}), X_i} }\,.
\end{align*}
Then, following the steps of the proof of \cref{apx:lemma:termI_bound_qnorm_vdel}, we get
\begin{align*}
\norm{\Vdel}_q 
& \le 
 \frac{2}{n^2} \sum_{i=1}^n \norm{\loss{ \Alg(\sS_n ) , X_i}}_{2q}^2 +
  \frac{2}{n} 
\sum_{\overset{i,j=1}{i\ne j}}^n
 \norm{\loss{ \Alg(\sS_n ) , X_i}  -  \loss{\Alg(\sS_n^{-j}), X_i} }_{2q}^2
\end{align*}
Now, notice that
\begin{align*}
\frac1n \sum_i \norm{\loss{ \Alg(\sS_n ) , X_i}}_{2q}^2
\le
\underbrace{\frac1n \sum_i \norm{\loss{ \Alg(\sS_n ) , X_i}-\loss{\Alg(\sS_n^{-i}),X_i} }_{2q}^2}_{\gamma_{2q}^2(n)}
+
\frac1n \sum_i \norm{\loss{\Alg(\sS_n^{-i}),X_i} }_{2q}^2\,.
\end{align*}
Furthermore,
\begin{align*}
\MoveEqLeft
 \norm{\loss{ \Alg(\sS_n ) , X_i}  -  \loss{\Alg(\sS_n^{-j}), X_i} }_{2q} \\
& \le
 \norm{\loss{ \Alg(\sS_n ) , X_i}  -  \loss{ \Alg(\sS_n^{-i}) , X_i} }_{2q}
 +
 \norm{\loss{\Alg(\sS_n^{-j}), X_i} - \loss{ \Alg(\sS_n^{-i} ) , X_i} }_{2q}\,.
\end{align*}
For the second term in the RHS we have
\begin{align*}
\MoveEqLeft \norm{\loss{\Alg(\sS_n^{-j}), X_i} - \loss{ \Alg(\sS_n^{-i} ) , X_i} }_{2q}\\
& \le
\norm{\loss{\Alg(\sS_n^{-j}), X_i} -  \loss{\Alg(\sS_n^{-i,-j}),X_i} }_{2q}
+
\norm{\loss{ \Alg(\sS_n^{-i} ) , X_i} -  \loss{\Alg(\sS_n^{-i,-j}),X_i} }_{2q}\,.
\end{align*}

Combining these inequalities and using $(a+b+c)^2 \le 3 (a^2+b^2+c^2)$, we get
\begin{align*}
\sum_{i\ne j}
\MoveEqLeft \norm{\loss{ \Alg(\sS_n ) , X_i}  -  \loss{\Alg(\sS_n^{-j}), X_i} }_{2q}^2 
 \le
3 
\sum_{i\ne j}
\norm{\loss{ \Alg(\sS_n ) , X_i}  -  \loss{ \Alg(\sS_n^{-i}) , X_i} }_{2q}^2 \\
& \quad +
3 \sum_{i\ne j}
\norm{\loss{\Alg(\sS_n^{-j}), X_i} -  \loss{\Alg(\sS_n^{-i,-j}),X_i} }_{2q}^2\\
& \quad +
3 \sum_{\overset{i,j=1}{i\ne j}}^n
\norm{\loss{ \Alg(\sS_n^{-i} ) , X_i} -  \loss{\Alg(\sS_n^{-i,-j}),X_i} }_{2q}^2\\
& = 3 n^2 \gamma_{2q}^2(n) + 3 n^2 \gamma_{2q}^2(n-1) + 3 n^2 \beta_{2q}^2(n-1)\,.
\end{align*}
Putting things together,
\begin{align*}
\norm{\Vdel}_q 
& \le 
 \frac{2}{n^2}
 \sum_{i=1}^n \norm{\loss{ \Alg(\sS_n^{-i}) , X_i}}_{2q}^2
+
  6n 
  ( \gamma_{2q}^2(n) + \gamma_{2q}^2(n-1) +  \beta_{2q}^2(n-1) )
  +
 \frac{2}{n} \gamma_{2q}^2(n)\,.
\end{align*}
Replacing $q$ with $2q$ we get the analogue of \cref{eq:termI_bound_qnorm_vdel}.
It follows that under \cref{assump:subgammagamma} (in place of \cref{assump:one}),
\cref{cor:term_1} and \cref{lemma:termI_rdel_exp_tail_bound} will hold with $\Rdel$ replaced by $\Rres$,
and $\beta_q$ replaced with $\gamma_q$, but with no other changes.

The second term of the RHS of \eqref{eq:basicdecres} is controlled and the derivations here are applicable given that 
\cref{assump:subgammagamma} implies \cref{assump:one}.
Previously, the third term was controlled in \cref{subsec:upper_bound_termIII}. Here, we need to change the reasoning a bit. We start by noting that

\begin{align*}
\E\Rres\bra{\Alg,\sS_n} - \E R\bra{\Alg(\sS_n),\cP}
=
\frac1n \sum_{i=1}^n \E\left[ \loss{\Alg(\sS_n),X_i} - \loss{\Alg(\sS_n),X}  \right]\,,
\end{align*}
where $X\sim \cP$ is independent of $\sS_n$.
Define $S_n^{i\setminus x} = (X_1,\dots,X_{i-1},x,X_{i+1},\dots,X_n)$.
Then,
$ \E \loss{\Alg(\sS_n),X_i} =  \E \loss{\Alg(S_n^{i\setminus X} ),X}$ and hence
\begin{align*}
\E\left[ \loss{\Alg(\sS_n),X_i} - \loss{\Alg(\sS_n),X}  \right]
= 
\E\left[ \loss{\Alg(S_n^{i\setminus X} ),X} - \loss{\Alg(\sS_n),X}  \right]\,.
\end{align*}
Now, taking absolute values, using $\abs{\E[V]} \le \E[\abs{V}]$,
subtracting and adding $\loss{\Alg(\sS_n^{-i}),X}$, and using the triangle inequality,
\begin{align*}
\MoveEqLeft 
\E \left[\abs{\loss{\Alg(S_n^{i\setminus X} ),X} - \loss{\Alg(\sS_n),X} }\right]\\
&  \le
\E \left[\abs{\loss{\Alg(S_n^{i\setminus X} ),X} - \loss{\Alg(\sS_n^{-i}),X} }\right]
+
\E \left[\abs{\loss{\Alg(\sS_n),X}- \loss{\Alg(\sS_n^{-i}),X}  } \right]\\
& =
\E \left[\abs{\loss{\Alg(S_n),X_i} - \loss{\Alg(\sS_n^{-i}),X_i} }\right]
+
\E \left[\abs{\loss{\Alg(\sS_n),X}- \loss{\Alg(\sS_n^{-i}),X}  } \right]\,.
\end{align*}
where the equality follows because 
the joint distribution of $(S_n^{i\setminus X},\sS_n^{-i},X)$ is the same as that of $(S_n,S_n^{-i},X_i)$.
Putting things together, we get
\begin{align*}
\MoveEqLeft
\abs{\E\Rres\bra{\Alg,\sS_n} - \E R\bra{\Alg(\sS_n),\cP}}\\
& \le 
\frac1n \sum_{i=1}^n
\E \left[\abs{\loss{\Alg(S_n),X_i} - \loss{\Alg(\sS_n^{-i}),X_i} }\right]
+
\frac1n \sum_{i=1}^n
\E \left[\abs{\loss{\Alg(\sS_n),X}- \loss{\Alg(\sS_n^{-i}),X}  } \right] \\
& \le  \gamma_1(n) + \beta_1(n)\,,
\end{align*}
where the last inequality is by Cauchy-Schwartz.
Combining this with the bounds on the other two terms in \eqref{eq:basicdecres} gives the desired result.
\end{proof}

\section{Proof of \cref{prop:srknn_lq_stability}}
\label{appx:example_short_range_1nn}

\newcommand{\one}[1]{\mathbf{I}\bra{#1}}
\newcommand{\nn}{\text{NN}}

By reusing Example~3.11 of \citet{kutin_niyogi_02}, we show that the following holds:

\propsrknnlqstability*

\begin{proof}
We consider classification of points of the $[0,1]$ interval with the zero-one loss.
Let the response be $\pm 1$. 
The loss is given by $\ell( g, (x,y) ) = \one{g(x)\ne y}$, where $(x,y)\in [0,1]\times \{\pm 1\}$ and $g: [0,1]\to \{\pm 1\}$. 
Choose the distribution $\cP$ so that the marginal on the input is the uniform distribution:
If $(X,Y) \sim \cP$, $\Prob{X\in [a,b]} = |b-a|$ for $0\le a\le b \le 1$.
Let $\cP$ be such that $\eta :=\Prob{Y=1}>0$.

The learning algorithm is what one may want to call a ``short-range nearest neighbor classifier''.
For $n\ge 1$, let $d_n \ge 0$ be the ``range'' parameter and assume that 
$d_n = o(1/n)$.
For data $\sS_n= ((X_1,Y_1),\dots,(X_n,Y_n))$, let $\nn( \sS_n, x ) = \argmin_{1\le i \le n} |X_i-x|$ be the index of the nearest neighbor of $x\in [0,1]$ in $\sS_n$ (with arbitrary tie-breaking).
Now, for input $x$ the learning algorithm $\Alg$ predicts label $Y_{\nn(\sS_n,x)}$ if $|\nn(\sS_n,x)-x|\le d_n$ and it predicts label $+1$ otherwise.
Clearly,
\begin{align}
\Rres\bra{\Alg,\fset{S}_n} =0~. \label{eq:emprisk}
\end{align}
Further,
\begin{align*}
\Prob{\Alg(\sS_n)(X)=-1\,|\,\sS_n} 
& \le 
\Prob{\,|\nn(\sS_n,X)-X|\le d_n\,|\,\sS_n} \\
& \le
\sum_{i=1}^n \Prob{\, |X-X_i|\le d_n \,|\, \sS_n } \\
& \le 2n d_n\,,
\end{align*}
where the last inequality follows because $X$ is uniformly distributed.

Note that $\{ \Alg(\sS_n)(X)\ne Y \} \supset \{ Y = -1 \} \setminus \{ \Alg(\sS_n)(X)=-1 \}$, hence, using~\eqref{eq:emprisk} and that by the independence of $(X,Y)$ and $\sS_n$, $\Prob{Y=-1|\sS_n}=\Prob{Y=-1}$,
\begin{align*}
R\bra{\Alg(\sS_n),\cP} - 
\Rres\bra{\Alg,\fset{S}_n}
& \ge \Prob{Y=-1} - \Prob{\Alg(\sS_n)(X)=-1|\sS_n} \\
& \ge \eta - 2 n d_n \to \eta \text{ as }  n\to\infty\,,
\end{align*}
which establishes \cref{eq:posgap}.
It remains to show that $\beta_q(n) \to 0$.
From the definition, since $X$ has a density, the rule is almost surely symmetric, we need to evaluate
\begin{align*}
\beta_q^q(n)=\expectone{ \left|\loss{\Alg(\sS_n),(X,Y)} - \loss{\Alg(\sS_n^{-1}),(X,Y)}\right|^q}\,,
\end{align*}
where $(X,Y) \sim \cP$, independently of $\sS_n$.
Note that 
$|\loss{\Alg(\sS_n),(X,Y)} - \loss{\Alg(\sS_n^{-1}),(X,Y)}|$ is either zero or one.
Let us focus on the case when this difference is nonzero, i.e., the two losses are different.
Clearly, if $\Alg(\sS_n)(X)=\Alg(\sS_n^{-1})(X)$ then the two losses were the same (both predictions are compared to the same $Y$). Hence, if the loss is nonzero, the predictions must be different.
It follows that
\begin{align*}
\beta_q^q(n) 
& = \Prob{ \loss{\Alg(\sS_n),(X,Y)} \ne \loss{\Alg(\sS_n^{-1}),(X,Y)} } 
\le \Prob{ \Alg(\sS_n)(X)\ne \Alg(\sS_n^{-1})(X) } \\
&=\EE\left[\, \Prob{ \Alg(\sS_n)(X)\ne \Alg(\sS_n^{-1})(X) \,|\, \sS_n } \,\right]\,.
\end{align*}
Now, $\Alg(\sS_n)(x)\ne \Alg(\sS_n^{-1})(x)$ implies that $|x-X_1|\le d_n$.
Since $X$ is uniformly distributed, independently of $\sS_n$,
$\Prob{ \Alg(\sS_n)(X)\ne \Alg(\sS_n^{-1})(X) \,|\, \sS_n } \le
\Prob{ |X-X_1|\le d_n \,|\, \sS_n } \le 2d_n$.
Hence,
\begin{align*}
\beta_q^q(n)  &  \le 2d_n
\end{align*}
and as such, $\beta_q(n) \le \sup_{q\ge 1} (2 d_n)^q/q \to 0$ as $n\to \infty$, as a simple calculation shows.
\end{proof}
Note that $\beta_u=1$ in this example: For $\sS_n$ such that $|X_1-X_2|<d_{n-1}$, $Y_1\ne Y_2$,
$\Alg(\sS_n)(X_1)\ne \Alg(\sS_n^{-1})(X_1)$, and 
\begin{align*}
|\loss{ \Alg(\sS_n), (X_1,y) }-\loss{ \Alg(\sS_n^{-1}), (X_1,y) }| =1\,.
\end{align*}

Also, short-range nearest neighbor is of course a terrible algorithm. Nevertheless if one is looking for ways of figuring out whether an algorithm is terrible or not, this example is important in that it shows that one should better look beyond  the empirical error. While this seems obvious in retrospect, we have not seen this mentioned in previous literature on comparing error estimation techniques. 

Note also that the $1$-nearest neighbor rule is also subject to the same phenomenon as the ``short-range nearest neighbor rule'': Its training error is always zero, while its risk converges to twice the Bayes risk. At the same time, it is stable in the above sense, and the deleted estimate concentrates around its true risk \citep{book_devroye_gyorfi_lugosi_1996}.

\section{\cref{lemma:termI_bound_qnorm_vdel} with $V$}
\label{sec:lemma:termI_bound_qnorm_v}
Let $( X_1',\dots, X_n')$ be an independent copy of $S_n$,
and 
\begin{align*}
\sS_n^i = (X_1,\dots,X_{i-1}, X_i', X_{i+1},\dots,X_n)
\end{align*}
 be the same as $S_n$ but with the $i$th example $X_i$ replaced by $\hat X_i$.
Define
\begin{align}
Z     & =  \Rdel(\Alg,\sS_n) = \frac{1}{n}\sum_{j=1}^n \loss{ \Alg(\sS_n^{-j} ) , X_j}\,.
\label{eq:Z_for_Rdel}
\end{align}
Clearly, $Z = f(\sS_n)$ for an appropriate function $f$. 
Define
\begin{align}
Z_i = f(\sS_n^i)\,.
\label{eq:Zi_for_Rdel}
\end{align}
Note that $\sS_n^{1,-1} = \sS_n^{-1}$ (replacing, the removing the first element is the same as removing it). Then,
\begin{align}
Z_1 =  \frac{1}{n}
\loss{ \Alg(\sS_n^{-1}), X_1' } +
\frac1n \sum_{j=2}^n \loss{ \Alg(\sS_n^{1,-j} ) , X_j}
\label{eq:Z1_exp_Rdel}
\end{align}
and a similar identity holds for $Z_i$ with $i>1$.
\begin{restatable}{lemma}{lemtermIqnormboundV}
\label{lemma:termI_bound_qnorm_v}
Let $Z$, $Z_{i}$ be defined as in \cref{eq:Z_for_Rdel,eq:Zi_for_Rdel}, and let
$V = \sum_{i=1}^n \bra{ Z - Z_i }^2$.
Then for any real $q \ge 1/2$ and integer $n \ge 2$, the following holds
\begin{align}
\norm{V}_{2q} & 
~ \le ~ 
 \sfrac{8}{n^2} \sum_{i=1}^n \norm{\loss{ \Alg(\sS_n^{-i} ) , X_i}}_{4q}^2 + 8 n \beta_{4q}^2(n-1)\,,
\label{eq:termI_bound_qnorm_v}
\end{align}
and, in particular, $\EE V ~ \le ~  \sfrac{8}{n^2}\sum_{i=1}^n\norm{\loss{ \Alg(\sS_n^{-i} ) , X_i}}_{2}^2 + 8 n\beta_{2}^2(n-1)$.
\end{restatable}
\begin{proof}
Let $q \ge 1$. Then,
\begin{align}
\norm{V}_q
& =
\norm{ \sum_{i=1}^n (Z - Z_{i})^2 }_q
 \le
\sum_{i=1}^n \norm{ (Z - Z_{i})^2 }_q
\,,
\label{eq:vbasic}
\end{align}
where the inequality is by the triangle inequality.
Now, using the definition of $Z$ (cf.  \cref{eq:Z_for_Rdel}) and identity \cref{eq:Z1_exp_Rdel} for $Z_1$,
\begin{align*}
\MoveEqLeft (Z - Z_{1})^2  \\
& = 
\bra{\frac{1}{n}  \bra{\loss{ \Alg(\sS_n^{-1} ) , X_1}- \loss{ \Alg(\sS_n^{-1} ) , X_1'}}
+
\frac{1}{n}\sum_{i=2}^n \bra{\loss{ \Alg(\sS_n^{-i} ) , X_i}  -  \loss{\Alg(\sS_n^{1 , -i}), X_i} }}^2 
\\
& 
\le
\frac{2}{n^2} \bra{ \loss{ \Alg(\sS_n^{-1} ) , X_1} - \loss{ \Alg(\sS_n^{-1} ) , X_1'} }^2\\
& \quad
+
2 \bra{\frac{1}{n}\sum_{i=2}^n \bra{\loss{ \Alg(\sS_n^{-i} ) , X_i}  -  \loss{\Alg(\sS_n^{1 , -i}), X_i} }}^2 
\tag{\mathtxtsm{$(a+b)^2 \le 2a^2 + 2b^2$}}
\\
& 
\le
\frac{4}{n^2}  \bra{ \losssq{ \Alg(\sS_n^{-1} ) , X_1} + \losssq{ \Alg(\sS_n^{-1} ) , X_1'} } \\
& \quad
+
2\bra{
\frac{1}{n}\sum_{i=2}^n \bra{\loss{ \Alg(\sS_n^{-i} ) , X_i}  -  \loss{\Alg(\sS_n^{1 , -i}), X_i} }}^2 
\tag{\mathtxtsm{$(a+b)^2 \le 2a^2 + 2b^2$ again}}
\\
& 
\le
\frac{4}{n^2}  \bra{ \losssq{ \Alg(\sS_n^{-1} ) , X_1} + \losssq{ \Alg(\sS_n^{-1} ) , X_1'} } \\
& \quad
+
\frac{2}{n}\sum_{i=2}^n \bra{\loss{ \Alg(\sS_n^{-i} ) , X_i}  -  \loss{\Alg(\sS_n^{1 , -i}), X_i} }^2 \,.
\quad\tag{\mathtxtsm{Jensen's inequality}}
\end{align*}
Taking the $q$-norm of both sides, 
using that $(S_n^{-1},X_1)$ and $(\sS_n^{-1},X_1')$ are identically distribution,
the triangle inequality and that for any $U$ RV, $\norm{U^2}_q = \norm{U}_{2q}^2$,
we get
\begin{align*}
\norm{ (Z - Z_1)^2 }_q 
& \le
\frac{8}{n^2} \norm{\losssq{ \Alg(\sS_n^{-1} ) , X_1}}_{2q}^2 + 
 \frac{2}{n}\sum_{i=2}^n \norm{\loss{ \Alg(\sS_n^{-i} ) , X_i}  -  \loss{\Alg(\sS_n^{1 , -i}), X_i} }_{2q}^2\,.
\end{align*}
Now, by the triangle inequality, using that $\sS_n^{1,-i,-1} =\sS_n^{1,-1-i} = \sS_n^{-1,-i}$,
\begin{align*}
\norm{\loss{ \Alg(\sS_n^{-i} ) , X_i}  -  \loss{\Alg(\sS_n^{1 , -i}), X_i} }_{2q} 
& \le
\norm{\loss{ \Alg(\sS_n^{-i} ) , X_i}  -  \loss{\Alg(\sS_n^{-1, -i}), X_i} }_{2q} \\
&  + {}
\norm{\loss{\Alg(\sS_n^{1 , -i}), X_i} - \loss{ \Alg(\sS_n^{1,-1,-i} ) , X_i} }_{2q}\,.
\end{align*}
Notice that the distribution of $\sS_n^{1}$ is identical to that of $\sS_n$.
Hence, the two terms on the right-hand side are equal and thus
\begin{align*}
\norm{\loss{ \Alg(\sS_n^{-i} ) , X_i}  -  \loss{\Alg(\sS_n^{1 , -i}), X_i} }_{2q}^2
\le
4 \norm{\loss{ \Alg(\sS_n^{-i} ) , X_i}  -  \loss{\Alg(\sS_n^{-1, -i}), X_i} }_{2q}^2
\end{align*}
and
\begin{align*}
\norm{ (Z - Z_1)^2 }_q 
& \le
\frac{8}{n^2} \norm{\losssq{ \Alg(\sS_n^{-1} ) , X_1}}_{2q}^2 + 
 \frac{8}{n}\sum_{i=2}^n \norm{\loss{ \Alg(\sS_n^{-i} ) , X_i}  -  \loss{\Alg(\sS_n^{-1 , -i}), X_i} }_{2q}^2\,.
\end{align*}
Note that the RHS here is exactly four times the RHS of the second term in \eqref{eq:RdelZminusZminusone} in the proof of \cref{lemma:termI_bound_qnorm_vdel}. Hence, the proof is finished as there:
First, note that an analogous inequality holds for $\norm{(Z-Z_j)^2}_q$ with $j>1$. 
Summing up all these, 
using that $(\Alg(\sS_n^{-j} ) , X_j)_{j}$ share the same distribution
and combining
with \cref{eq:vbasic}, we get
\begin{align*}
\norm{V}_q
 & \le \frac{8}{n^2} \sum_{j=1}^n \norm{\loss{ \Alg(\sS_n^{-j} ) , X_j}}_{2q}^2 +
  8 \frac{1}{n} \sum_{j=1}^n
\sum_{i\ne j} \norm{\loss{ \Alg(\sS_n^{-i} ) , X_i}  -  \loss{\Alg(\sS_n^{-j , -i}), X_i} }_{2q}^2\\
 & = \frac{8}{n^2} \sum_{j=1}^n  \norm{\loss{ \Alg(\sS_n^{-j} ) , X_j}}_{2q}^2 +
  8 \sum_{i=1}^n
\underbrace{\frac{1}{n} \sum_{j\ne i}  
 \norm{\loss{ \Alg(\sS_n^{-i} ) , X_i}  -  \loss{\Alg(\sS_n^{-i , -j}), X_i} }_{2q}^2}_{\beta_{2q}^2(n-1)}\\
& = \frac{8}{n^2} \sum_{i=1}^n \norm{\loss{ \Alg(\sS_n^{-i} ) , X_i}}_{2q}^2 + 8 n \beta_{2q}^2(n-1)\,.
\end{align*}
Replacing $q$ with $2q$ gives the desired result.
\end{proof}

\section{\cref{lemma:temrII_vdel_is_subgamma} with $V$}
\label{sec:lemma:temrII_v_is_subgamma}
Let $\sS_n'$ and $\sS_n^i$ be as in the previous section.
Define
\begin{equation}
Z  =  R\bra{\Alg(\sS_n),\fdistri{P}}\,.
\label{eq:Z_for_risk}
\end{equation}
Then, $Z = f(\sS_n)$ for some function $f$.
Define
\begin{align}
Z_i = f(\sS_n^i)\,.
\label{eq:Zi_for_risk}
\end{align}
Similar to \cref{lemma:temrII_vdel_is_subgamma} we have the following result:
\begin{restatable}{lemma}{lemtermIIqnormboundv}
\label{lemma:temrII_v_is_subgamma}
\quad
Let $Z$ and $Z^i$ be defined as in \eqref{eq:Z_for_risk} and \eqref{eq:Z_for_risk}, respectively, and let
$V = \sum_{i=1}^n(Z - Z_i)^2$.
Then for any real $q \ge 1/2$, and $n \ge 2$, the following holds
\begin{align}
\norm{V}_{2q}
& ~ \le ~
4 n \beta^2_{4q}(n) 
~,
\label{eq:termII_expect_v_upper_bound}
\end{align}
and, in particular, $\EE V ~ \le ~ 4 n\beta_2^2(n)$.
\end{restatable}

\begin{proof}
Let $q\ge 1$. Then, similar to the previous proofs,
\begin{align}
\norm{V}_q
=
\norm{ \sum_{i=1}^n \bra{Z - Z_i}^2 }_q
\le
\sum_{i=1}^n
\norm{ \bra{Z - Z_i}^2 }_q
= 
\sum_{i=1}^n
\norm{ \bra{Z - Z_i} }_{2q}^2\,.
\label{eq:v2exp}
\end{align}
Then, using the definitions of $Z$ and $Z_1$, with the same calculation as in \eqref{ineq:termII_moment_vdel_final}, just replacing $\sS_n^{-1}$ everywhere with $\sS_n^1$,
\begin{align}
\norm{Z-Z_1}_{2q}^2
& \le
\norm{
R(\Alg(\sS_n),\fdistri{P}) - R(\Alg(\sS_n^{1}),\fdistri{P})
}_{2q}^2
\nonumber
 \le
\norm{
\loss{\Alg(\fset{S}_n),X} - \loss{\Alg(\fset{S}_n^{1}),X}
}_{2q}^{2}\,.
\nonumber
\nonumber
\end{align}
Now, using that $\sS_n^{1,-1} = \sS_n^{-1}$ and that $(\sS_n^1,X)$ and $(\sS_n,X)$ are identically distributed,
\begin{align*}
\norm{
\loss{\Alg(\fset{S}_n),X} - \loss{\Alg(\fset{S}_n^{1}),X}
}_{2q}
& \le
\norm{
\loss{\Alg(\fset{S}_n),X} - \loss{\Alg(\fset{S}_n^{-1}),X}
}_{2q}\\
& +{}
\norm{
\loss{\Alg(\fset{S}_n^1),X} - \loss{\Alg(\fset{S}_n^{1,-1}),X}
}_{2q} \\
& =
2 \norm{
\loss{\Alg(\fset{S}_n),X} - \loss{\Alg(\fset{S}_n^{-1}),X}
}_{2q}\,.
\end{align*}
An analogous inequality holds for $\norm{Z-Z_{-i}}_{2q}^2$ with $i>1$. Summing up all these,
combining with \eqref{eq:v2exp} we get
\begin{align*}
\norm{V}_q \le n\,\, \frac{4}{n} \sum_{i=1}^n \norm{
\loss{\Alg(\fset{S}_n),X} - \loss{\Alg(\fset{S}_n^{-i}),X}
}_{2q}^{2} = 4 n \beta_{2q}^2(n)~.
\end{align*}
Replacing $q$ with $2q$ yields that
\begin{align}
\norm{V}_{2q}
& ~ \le ~
4 n \beta^2_{4q}(n) ~.
\nonumber
\end{align}
\end{proof}

\section{\cref{theorem:cgf_z_vdel_expbound} for $V$}
\label{sec:theorem:cgf_z_v_expbound}

\begin{restatable}{theorem}{thmcgfzvexpbound}
\label{theorem:cgf_z_v_expbound}
Let $Z = f(\sS_n)$, $Z_i = f(\sS_n^i)$, $V = \sum_{i=1}^n (Z-Z_i)^2$ as in the previous section.
Then, for all $\theta > 0$, s.t. $\theta\lambda < 1$, and $\expectop e^{\lambda V} < \infty$,
the following holds
\begin{equation}
\log\expectone{\exp\bra{ \lambda (Z - \expectop Z) }}
\le
\sfrac{\lambda\theta}{(1 - \lambda \theta)}
\log\expectone{\exp\bra{\sfrac{\lambda V}{\theta}}}. \label{eq:cgf_z_v}
\end{equation}
\end{restatable}
\begin{proof}
Theorem~2 of \citet{boucheron_lugosi_massart_2003} states that under the assumed conditions,
\begin{equation}
\log\expectone{\exp\bra{ \lambda (Z - \expectop Z) }}
\le
\sfrac{\lambda\theta}{(1 - \lambda \theta)}
\log\expectone{\exp\bra{\sfrac{\lambda V_+}{\theta}}}\,,
\label{eq:vpluscgfbound}
\end{equation}
where $V_+ = \expectone{ \sum_{i=1}^n (Z-Z_i)^2 \one{Z>Z_i} |\sS_n}$.
Now, 
\begin{align*}
\E \exp\bra{\sfrac{\lambda V_+}{\theta}}
& =
\E \exp\bra{\sfrac{\lambda  \expectone{ \sum_{i=1}^n (Z-Z_i)^2 \one{Z>Z_i} |\sS_n}}{\theta}} \\
& \le
\expectone{ \expectone{ \exp\bra{ \sfrac{\lambda  \sum_{i=1}^n (Z-Z_i)^2 \one{Z>Z_i} }{\theta}} | \sS_n } }
 \tag{Jensen's ineq.}
 \\
& =
\E \exp\bra{  \sfrac{\lambda  \sum_{i=1}^n (Z-Z_i)^2 \one{Z>Z_i} }{\theta}  } 
\tag{tower rule} 
\\
& =
\E \exp\bra{  \sfrac{\lambda  \sum_{i=1}^n (Z-Z_i)^2 }{\theta}  } 
\tag{$(Z-Z_i)^2\ge 0$}
\\
& =
\E \exp\bra{  \sfrac{\lambda  V }{\theta}  } \,.
\end{align*}
Plugging this into \eqref{eq:vpluscgfbound}  we get \eqref{eq:cgf_z_v}.
\end{proof}
The inequality for the lower tail is obtained the same way.

\end{document}